\documentclass[10pt]{article} 
\usepackage[accepted]{rlc}

\usepackage{amssymb}            
\usepackage{mathtools}          
\usepackage{mathrsfs}           
\usepackage{graphicx}           
\usepackage{subcaption}         
\usepackage[space]{grffile}     
\usepackage{url}                
\usepackage{amsthm}
\usepackage{algorithm}
\usepackage{algpseudocode} 
\usepackage{booktabs}
\usepackage[export]{adjustbox}
\usepackage{wrapfig}
\usepackage{makecell}

\definecolor{mydarkblue}{rgb}{0,0.08,0.45}
\definecolor{myred}{rgb}{0.84,0.17,0.11}
\definecolor{mygreen}{rgb}{0.35,0.60,0.25}
\definecolor{myblue}{rgb}{0.19,0.44,0.72}

\theoremstyle{plain}
\newtheorem{theorem}{Theorem}[section]
\newtheorem{proposition}[theorem]{Proposition}
\newtheorem{lemma}[theorem]{Lemma}

\theoremstyle{definition}
\newtheorem{definition}[theorem]{Definition}
\newtheorem{assumption}[theorem]{Assumption}
\theoremstyle{remark}

\title{Bidirectional-Reachable Hierarchical Reinforcement Learning with Mutually Responsive Policies}

\author{Yu Luo  \\
    luoyu19@mails.tsinghua.edu.cn \\
    Department of Computer Science and Technology\\
    Tsinghua University
    \And
    Fuchun Sun$^\star$  \\
    fcsun@tsinghua.edu.cn \\
    Department of Computer Science and Technology\\
    Tsinghua University
    \And
    Tianying Ji  \\
    jity20@mails.tsinghua.edu.cn \\
    Department of Computer Science and Technology\\
    Tsinghua University
    \And
    Xianyuan Zhan  \\
    zhanxianyuan@air.tsinghua.edu.cn \\
    Institute for AI Industry Research\\
    Tsinghua University}

\begin{document}

\maketitle

\begin{abstract}
Hierarchical reinforcement learning (HRL) addresses complex long-horizon tasks by skillfully decomposing them into subgoals. Therefore, the effectiveness of HRL is greatly influenced by subgoal reachability.
Typical HRL methods only consider subgoal reachability from the unilateral level, where a dominant level enforces compliance to the subordinate level.
However, we observe that when the dominant level becomes trapped in local exploration or generates unattainable subgoals, the subordinate level is negatively affected and cannot follow the dominant level's actions.
This can potentially make both levels stuck in local optima, ultimately hindering subsequent subgoal reachability.
Allowing real-time bilateral information sharing and error correction would be a natural cure for this issue, which motivates us to propose a mutual response mechanism.
Based on this, we propose the Bidirectional-reachable Hierarchical Policy Optimization~(BrHPO)—a simple yet effective algorithm that also enjoys computation efficiency.
Experiment results on a variety of long-horizon tasks showcase that BrHPO outperforms other state-of-the-art HRL baselines, coupled with a significantly higher exploration efficiency and robustness\footnote{We have released our code here: \url{https://github.com/Roythuly/BrHPO}}.
\end{abstract}

\section{Introduction}
Reinforcement learning (RL) has demonstrated impressive capabilities in decision-making scenarios, ranging from achieving superhuman performance in games~\citep{mnih2015human,lample2017playing,silver2018general}, developing complex skills in robotics~\citep{levine2016end, schulman2015high} and enabling smart policies in autonomous driving~\citep{jaritz2018end,kiran2021deep,cao2023continuous}.
Most of these accomplishments are attributed to single-level methods~\citep{sutton2018reinforcement}, which learn a flat policy by trial and error without extra task decomposition or subgoal guidance.
While single-level methods excel at short-horizon tasks involving inherently atomic behaviors~\citep{levy2018learning,nachum2018data,pateria2021hierarchical}, they often struggle to optimize effectively in long-horizon complex tasks that require multi-stage reasoning or sparse reward signals.
To address this challenge, hierarchical reinforcement learning (HRL) has been proposed, aiming to decompose complex tasks into a hierarchy of subtasks or skills~\citep{kulkarni2016hierarchical, bacon2017option, vezhnevets2017feudal}. 
By exploiting subtask structure and acquiring reusable skills, HRL empowers agents to solve long-horizon tasks efficiently.

Subgoal-based HRL methods, a prominent paradigm in HRL, partition complex tasks into simpler subtasks by strategically selecting subgoals to guide exploration~\citep{vezhnevets2017feudal, nachum2018data}.
Subgoal reachability, which is utilized as an intrinsic reward for exploration in different subtasks~\citep{sukhbaatar2018intrinsic}, is crucial in evaluating how effectively the low-level policies' exploration trajectory aligns with the high-level policy's subgoal, ultimately determining task performance~\citep{vezhnevets2017feudal, zhang2020generating}.
However, existing approaches for improving subgoal reachability predominantly focus on one level of the hierarchical policy, imposing dominance on the other level. This can be categorized as either low-level dominance or high-level dominance~\citep{nachum2018data, zhang2020generating, andrychowicz2017hindsight, chane2021goal, eysenbach2019search, jurgenson2020sub}.
Low-level dominance (Figure~\ref{structure_figure}\textcolor{mydarkblue}{a}) refers to the accommodation of low-level passive exploratory behaviour by the high-level policy, causing the agent to get stuck near the starting position. 
On the other hand, high-level dominance (Figure~\ref{structure_figure}\textcolor{mydarkblue}{b}) may result in unattainable subgoals, causing repeated failure and sparse learning signals for the low-level policy.
To assess these methods, we applied them to two HRL benchmarks, AntMaze and AntPush, and generated state-subgoal trajectories for visualization.
The results reveal that the former methods exhibit lower exploration efficiency as the high level must generate distant subgoals to guide the low level (Figure~\ref{effectiveness_visualization}\textcolor{mydarkblue}{a}), while the latter methods may create unattainable subgoals, resulting in the low-level policy's inability to track them (Figure~\ref{effectiveness_visualization}\textcolor{mydarkblue}{b}). 

\begin{figure*}[t]
\centering
\vspace{-10px}
\begin{minipage}[t]{.49\textwidth}
    \centering
    \includegraphics[width=1.0\textwidth]{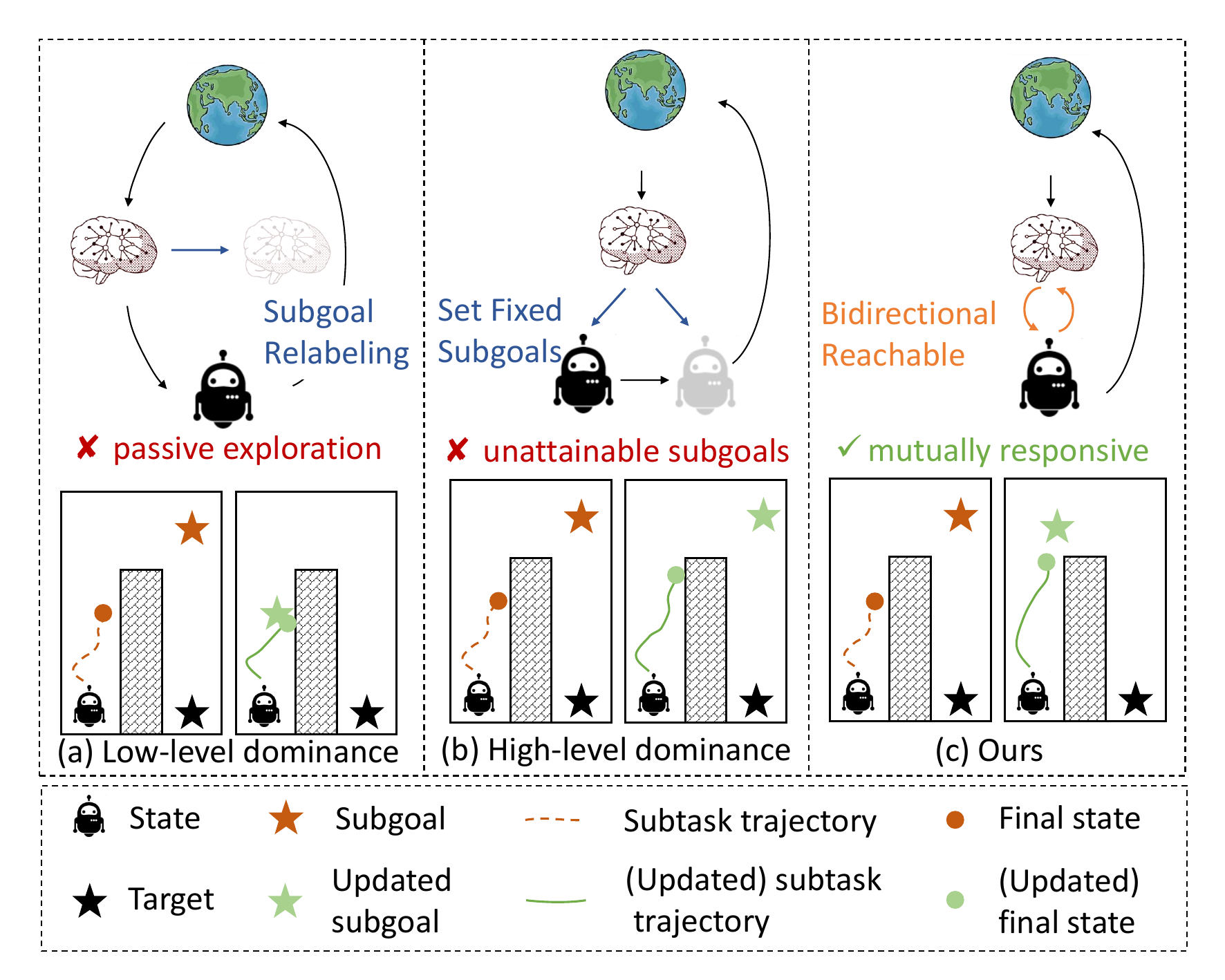}
    \caption{A motivating example of our proposed BrHPO. The earth, brain, and robot symbols stand for the environment, high-level policy, and low-level policy, respectively. We illustrate the behaviors of hierarchical policies before and after updated for each case. \textbf{Left}: Updated subgoal is limited by low-level exploration. \textbf{Middle}: Low-level policy struggles to approach the fixed subgoal. \textbf{Right}: hierarchical policies are mutually responsive for subgoal reachability.}
    \vspace{-10px}
    \label{structure_figure}
\end{minipage}
\hfill
\hspace{2pt}
\begin{minipage}[t]{.49\textwidth}
    \centering
    \includegraphics[width=1.0\textwidth]{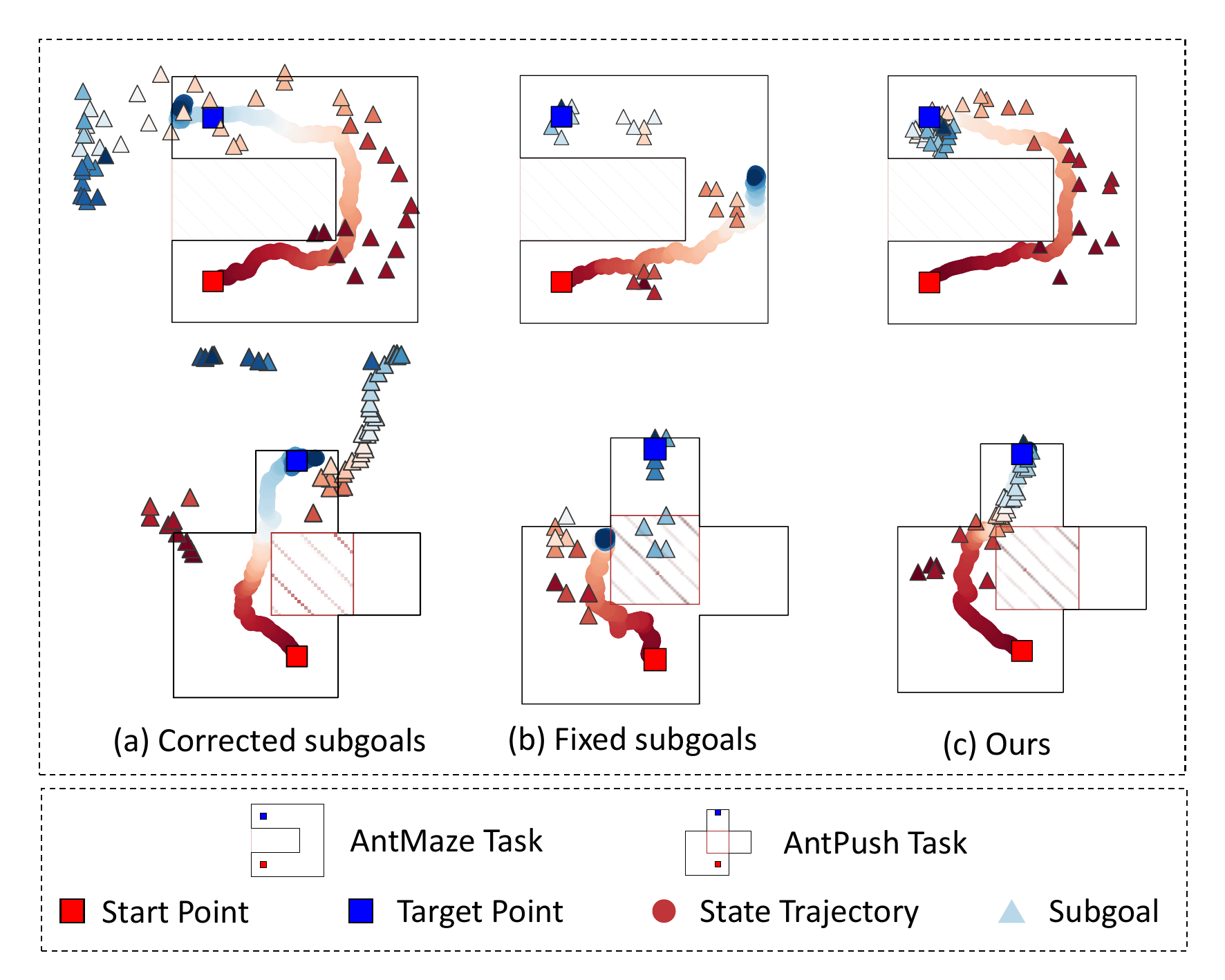}
    \caption{The state-subgoal trajectory comparison of baselines \textbf{HIRO} (a), \textbf{RIS} (b) and our \textbf{BrHPO} (c). We visualize the state trajectories (represented by the red-to-blue gradient lines) and the guided subgoals (represented by triangles). Note that lines and triangles of the same colour indicate that they belong to the same subtask. The results demonstrate that BrHPO can improve the alignment between states and subgoals, thus benefitting overall performance.}
    \vspace{-10px}
    \label{effectiveness_visualization}
\end{minipage}
\end{figure*}

Enforcing subgoal reachability through unidirectional communication between the two levels has limitations in overall performance improvement.
A bidirectional reachability approach, illustrated in Figure~\ref{structure_figure}\textcolor{mydarkblue}{c}, holds the potential to be more effective in HRL.
From an optimization perspective, bidirectional reachability provides two key benefits: 1) \textit{the high-level policy can generate subgoals that strike a balance between incentive and accessibility}, and 2) \textit{the low-level policy can take more effective actions that drive subtask trajectories closer to the subgoal}.
Despite its potential advantages, bidirectional subgoal reachability has not been extensively studied in previous research, and its effectiveness in enhancing HRL performance requires further investigation.
We explore the theoretical benefits of bidirectional insights, and empirically demonstrate its effectiveness through visualizing the alignment between states and subgoals in Figure~\ref{effectiveness_visualization} and our ablation studies.

This paper aims to investigate the potential of bidirectional subgoal reachability in improving subgoal-based HRL performance, both theoretically and empirically. 
Specially, we propose a joint value function and then derive a performance difference bound for hierarchical policy optimization.
The analysis suggests that enhancing subgoal reachability, from the mutual response of both-level policies, can effectively benefit overall performance. 
Motivated by these, our main contribution is a simple yet effective algorithm, Bidirectional-reachable Hierarchical Policy Optimization (BrHPO) which incorporates a mutual response mechanism to efficiently compute subgoal reachability and integrate it into hierarchical policy optimization.
Through empirical evaluation, we demonstrate that BrHPO achieves promising asymptotic performance and exhibits superior training efficiency compared to state-of-the-art HRL methods. 
Additionally, we investigate different variants of BrHPO to showcase the effectiveness and robustness of the proposed mechanism.

\section{Preliminaries}
We consider an infinite-horizon discounted Markov Decision Process (MDP) with state space $\mathcal{S}$, action space $\mathcal{A}$, goal/subgoal space $\mathcal{G}$, unknown transition probability $P_{s,s'}^{a}: \mathcal{S}\times\mathcal{A}\times\mathcal{S}\rightarrow[0,1]$, reward function $r:\mathcal{S}\times\mathcal{A}\times\mathcal{G}\rightarrow\mathbb{R}$, and discounted factor $\gamma\in(0,1)$. The objective of RL is to find a policy $\pi:\mathcal{S}\rightarrow\Delta(\mathcal{A})$ to maxmize the discounted cumulative reward from the environment, which can be formed as $\pi^*=\arg\max_{\pi}\mathbb{E}_{(s_t,a_t)\sim P,\pi}\left[\sum_{t=0}^\infty\gamma^tr(s_t,a_t)\right]$.

Subgoal-based HRL, also called Feudal HRL~\citep{dayan1992feudal,vezhnevets2017feudal}, comprises two hierarchies: a high-level policy generating subgoals, and a low-level policy pursuing subgoals in each subtask.
Assume that each subtask contains a fixed length of $k$ timesteps, allowing us to split the original task into multiple subtasks.
Given the task goal $\hat{g}$, at the beginning of the $i$-th subtask where $i\in\mathbb{N}$, the high-level policy $\pi_h$ observes state $s_{ik}$ and then outputs a subgoal $g_{(i+1)k}\sim\pi_h(\cdot|s_{ik},\hat{g})\in\mathcal{G}$.
Then, in each subtask, the low-level policy $\pi_l$ performs actions conditioned on the subgoal and the current state, $a_{ik+j}\sim\pi_l(\cdot|s_{ik+j},g_{(i+1)k})\in\mathcal{A}$, where $j\in[0,k-1]$ is a pedometer in one single subtask. 
With the guidance from the subgoal, the state-subgoal-action trajectory in the $i$-th subtask comes out to be
\begin{equation}
\tau_i^{\pi_h,\pi_l}\triangleq \left\{(s_{ik+j})|s_{ik},g_{(i+1)k}\sim\pi_h(\cdot|s_{ik},\hat{g}),a_{ik+j}\sim\pi_l(\cdot|s_{ik+j},g_{(i+1)k})\right\}_{j=0}^{k-1},
\end{equation}
and the whole task trajectory forms by stitching all subtask trajectories as $\tau=\cup_{i=0}^\infty\left(\tau_i^{\pi_h,\pi_l}\right)$. 

Following prior methods~\citep{andrychowicz2017hindsight,nachum2018data,zhang2020generating}, 
we optimize $\pi_h$ based on the high-level reward $r_h$, defined as the environment reward feedback summated over a subtask
\begin{equation}\label{highreward}
r_h(\tau_i^{\pi_h,\pi_l})=r_h(s_{ik},g_{(i+1)k})=\sum_{j=0}^{k-1}r(s_{ik+j},a_{ik+j}),
\end{equation}
and the intrinsic reward for the low-level policy $\pi_l$ is 
\begin{equation}\label{lowreward}
r_l(s_{ik+j},a_{ik+j},g_{(i+1)k})=-\mathcal{D}(\psi(s_{ik+j+1}),g_{(i+1)k}).
\end{equation}
where $\psi:\mathcal{S}\mapsto\mathcal{G}$ is a pre-defined state-to-goal mapping function and $\mathcal{D}:\mathcal{G}\times\mathcal{G}\rightarrow\mathbb{R}_{\geq 0}$ is a chosen binary or continuous distance measurement~\citep{zhang2022adjacency}.

\section{Bidirectional Subgoal Reachability in HRL}
In this section, we introduce the concept of bidirectional subgoal reachability and highlight its differences from the previously studied unidirectional reachability. Specifically, bidirectional subgoal reachability considers the capacities of both high-level guidance and low-level exploration, allowing for more flexibility in HRL. We then discuss how this bidirectional reachability is integrated into the optimization objective of hierarchical policies, resulting in a mutual response mechanism. Finally, we present performance difference bounds associated with bi-directional reachability, providing valuable theoretical insights for subgoal-based HRL.

\subsection{Bidirectional Subgoal Reachability}
In contrast to previous unilateral subgoal reachability, termed as the high- or low-level dominance, our work aims to propose a bidirectional subgoal reachability metric that simultaneously considers the cooperation capacities of high-level guidance and low-level exploration within a single subtask.
\begin{definition}
The bidirectional subgoal reachability $\mathcal{R}^{\pi_h,\pi_l}_i$ at the $i$-th subtask is defined by
\begin{equation}\label{reachability_function}
\mathcal{R}^{\pi_h,\pi_l}_i=\mathbb{E}_{g_{(i+1)k}\sim\pi_h,s_{(i+1)k}\sim\tau^{\pi_h,\pi_l}_i}\left[\mathcal{D}(\psi(s_{(i+1)k}),g_{(i+1)k})/\mathcal{D}(\psi(s_{ik}),g_{(i+1)k})\right].
\end{equation}
\end{definition}

In this definition, subgoal reachability is equal to the ratio of the final distance (the final reached state $s_{(i+1)k}$ to the subgoal $g_{(i+1)k}$) to the initial distance (the initial state $s_{ik}$ to the subgoal $g_{(i+1)k}$). Note that the smaller $\mathcal{R}^{\pi_h,\pi_l}_i$ means the better subgoal reachability and we define $\mathcal{R}^{\pi_h,\pi_l}_i=0$ if $\mathcal{D}(\psi(s_{ik}),g_{(i+1)k})=0$. Although conceptually simple, this form has two benefits:
\begin{itemize}
    \item When given the initial state $s_{ik}$ of the sub-task, the subgoal reachability depends only on the final distance, and is independent of the intermediate exploration process, aligning with the properties of hierarchical abstraction. Besides, this is conducive to decoupling the guidance of the high-level policy and the exploration of the low-level policy, avoiding the issues of high- or low-level dominances;
    \item Using initial distance $\mathcal{D}(\psi(s_{ik}),g_{(i+1)k})$ as the regularization can eliminate the difference caused by the initial conditions of different sub-tasks. Thus, it can comprehensively measure whether the subgoal is easily reachable and whether the sub-task is easy to complete. For instance, a subgoal with an initial distance of $10$ and a final distance of $3$, although the final distance is larger, has a better subgoal reachability than a subgoal with an initial distance of $5$ but a final distance $2$.
\end{itemize}

In contrast to previous methods, such as using environmental dynamics~\citep{zhang2020generating} or policy behavior~\citep{nachum2018data,kreidieh2019inter} for measuring subgoal reachability, our method is a continuous metric and can assess the cooperative effects of hierarchical policies rather than one of them. Therefore, improving subgoal reachability during policy optimization can be effective in enhancing the performance of hierarchical policies. Further, by recognizing that the low-level intrinsic reward shares the same form as the distance computation, we can replace the distance computation with the low-level reward. Thus, we can calculate the subgoal reachability by
\begin{equation}\label{appro_subgoal_reachability}
\mathcal{R}^{\pi_h,\pi_l}_i=\mathbb{E}_{g_{(i+1)k}\sim\pi_h,s_{(i+1)k}\sim\tau^{\pi_h,\pi_l}_i}\left[\frac{\mathcal{D}(\psi(s_{(i+1)k}),g_{(i+1)k})}{\mathcal{D}(\psi(s_{ik}),g_{(i+1)k})}\right]={\mathbb{E}_{r_l\sim\tau^{\pi_h,\pi_l}_i}\frac{r_{l,(i+1)k}}{r_{l,ik}}}.
\end{equation}
Specifically, we use a temporary replay buffer for storing subtask trajectory $\tau^{\pi_h,\pi_l}_i$ upon subtask completion. Then, we can sample the first low-level reward $r_{l,ik}=r_l(s_{ik},a_{ik},g_{(i+1)k})$ and the last one $r_{l,(i+1)k}=r_l(s_{(i+1)k},a_{(i+1)k},g_{(i+1)k})$ from the temporary buffer to calculate the reachability. Notably, such a design is quite lightweight, incurring $O(1)$ computational complexity, without introducing additional training costs.

\subsection{Bidirectional Reachability Hierarchical Policy Optimization}
With the bidirectional subgoal reachability in hand, we turn to design the core mutual response mechanism, which aims at enhancing the reachability with the help of hierarchical policies.

\paragraph{High-level policy optimization.}
In our approach, we opt to use $\mathcal{R}^{\pi_h,\pi_l}_i$ as a regularization for optimizing $\pi_h$. During the high-level policy evaluation phase, we exclusively rely on rewards from the environment to iteratively compute Q-values, which ensures the accuracy of guidance performance evaluation. Furthermore, in the policy improvement phase, using $\mathcal{R}^{\pi_h,\pi_l}_i$ as the regularization explicitly constrains the high-level policy's subgoal generation. This focus allows it to concern the subgoal-reaching performance of the low-level policy within a subtask. Let $D_h=D_h\cup\{\tau^{\pi_h,\pi_l}_i\}$ be the high-level replay buffer, we evaluate the high-level policy by,
\begin{equation}\label{critic_high}
Q^{\pi_h}(s,g)=\arg\min_Q\frac{1}{2}\mathbb{E}_{s,g\sim D_h}\left[r_h(s,g)+\gamma\mathbb{E}_{s'\sim D_h,g'\sim\pi_h}Q^{\pi_h}(s',g')-Q^{\pi_h}(s,g)\right]^2,
\end{equation}
and update the high-level policy by minimizing the expected KL-divergence with the reachability term as,
\begin{equation}\label{actor_high}
\pi_h=\arg\min_{\pi_h}\mathbb{E}_{s\sim D_h}\left[{\rm D}_{KL}(\pi_h(\cdot|s)\Vert \exp(Q^{\pi_h}(s,g)-V^{\pi_h}(s)))+\lambda_1\mathcal{R}^{\pi_h,\pi_l}_i\right],
\end{equation}
where $V^{\pi_h}(s)=\mathbb{E}_{g\sim\pi_h(\cdot|s)}\left[Q^{\pi_h}(s,g)-\log\pi_h(\cdot|s)\right]$ is the high-level soft state value function and $\lambda_1$ is a weight factor. Thus, we can adjust the response of the high level through tuning $\lambda_1$. 

\paragraph{Low-level policy optimization.}
In contrast to high-level policy, we utilize $\mathcal{R}^{\pi_h,\pi_l}_i$ as a reward bonus for low-level policy. This approach is designed to enable $\pi_l$ to simultaneously focus on both low-level rewards and subgoal reachability during subgoal exploration. To ensure the improvement of bidirectional subgoal reachability by low-level policy, we introduce subgoal reachability as well as the low-level reward, which is formulated as
\begin{equation}\label{modified_low_reward}
\hat{r}_l(s_{ik+j},a_{ik+j},g_{(i+1)k})=r_l(s_{ik+j},a_{ik+j},g_{(i+1)k})-\lambda_2\mathcal{R}^{\pi_h,\pi_l}_i.
\end{equation}
Let $D_l=D_l\cup\{(s,g,a,\hat{r}_l,s',g)\}$ be the low-level replay buffer. With the surrogate low-level reward established, the evaluation and optimization of low-level policy can be performed by
\begin{equation}\label{critic_low}
Q^{\pi_l}(s,a)=\arg\min_Q\frac{1}{2}\mathbb{E}_{s,g,a\sim D_l}\left[\hat{r}_l(s,a,g)+\gamma\mathbb{E}_{s',g\sim D_l,a'\sim\pi_l}Q^{\pi_l}(s',a')-Q^{\pi_l}(s,a)\right]^2,
\end{equation}
\begin{equation}\label{actor_low}
\pi_l=\arg\min_{\pi_l}\mathbb{E}_{s,g\sim D_l}\left[{\rm D}_{KL}(\pi_l(\cdot|s,g)\Vert \exp(Q^{\pi_l}(s,a)-V^{\pi_l}(s)))\right].
\end{equation}

\subsection{Theoretical Insights}
The previous subsection proposes an optimization algorithm for high- and low-level policies based on bidirectional subgoal reachability, and we investigate how this algorithm works in this section. First, to evaluate the overall performance of HRL, we construct a joint value function by calculating the discounted summation of step-wise rewards accumulated along the trajectory generated by both the high- and low-level policies, as presented below:
\begin{definition}[Joint Value Function of Hierarchical Policies]
The long-term cumulative return $V^{\pi_h,\pi_l}(s_0)$ of the subgoal-based HRL in the real environment can be defined as,
\begin{align}\label{Joint_Value_Function}
V^{\pi_h,\pi_l}(s_0) &= \sum_t^\infty\gamma^t\mathbb{E}_{s,a\sim \mathbb{P}^{\pi_l,g}_t(\cdot,\cdot|s_0), g\sim\pi_h(\cdot|s)}\left[r(s_t,a_t,\hat{g})\right]\nonumber\\
&=\sum_{i=0}^{\infty}\mathbb{E}_{g\sim \pi_h(\cdot|s)}\left[\gamma^{ik}\left(\sum_{j=0}^{k-1}\gamma^j\mathbb{E}_{s,a\sim \mathbb{P}^{\pi_l,g}_{ik+j}(\cdot,\cdot|s_0)}r(s_{ik+j},a_{ik+j},\hat{g})\right)\right].
\end{align}
\end{definition}

To investigate the optimality of the policies, we derive a performance difference bound between an induced optimal hierarchical policy $\Pi^*=\{\pi^*_h,\pi^*_l\}$ and a learned one $\Pi=\{\pi_h,\pi_l\}$, which can be formulated as $V^{\Pi^*}(s)-V^{\Pi}(s)\leq C$.
\begin{theorem}[Sub-optimal performance difference bound of HRL]\label{PDL_HRL}
The performance difference bound $C$ between the induced optimal hierarchical policies $\Pi^*$ and the learned one $\Pi$ can be
\begin{align}\label{PDL_bound_C} 
C(\pi_h,\pi_l)=
\frac{2r_{max}}{(1-\gamma)^2}\Bigg
[\underbrace{(1+\gamma)\mathbb{E}_{g\sim\pi_h}\left(1+\frac{\pi^*_h}{\pi_h}\right)\epsilon^g_{\pi^*_l,\pi_l}}_{\text{(\romannumeral1) hierarchical policies' inconsistency}}+\underbrace{2\left(\mathcal{R}^{\pi_h,\pi_l}_{max} + 2\gamma^k\right)}_{\text{(\romannumeral2) subgoal reachability penalty}}\Bigg],
\end{align}
where $\epsilon^g_{\pi^*_l,\pi_l}$ is the distribution shift between $\pi^*_l$ and $\pi_l$, and $\mathcal{R}^{\pi_h,\pi_l}_{max}$ is the maximum subgoal reachability penalty from the learned one $\Pi$, both of which are formulated as,
\begin{equation*}
\epsilon^g_{\pi^*_l,\pi_l}=\max_{s\in\mathcal{S},g\sim\pi_h}{\rm D}_{TV}\left(\pi^*_l(\cdot|s,g)\Vert\pi_l(\cdot|s,g)\right) \quad \text{and} \quad \mathcal{R}^{\pi_h,\pi_l}_{max}=\max_{i\in\mathbb{N}}\mathcal{R}^{\pi_h,\pi_l}_i.
\end{equation*}
\end{theorem}

Please refer to Appendix~\ref{appendix:omitter_proof} for the detailed proof. As shown in Equation~(\ref{PDL_bound_C}), the performance difference bound consists of two parts: (i) hierarchical policies' inconsistency and (ii) subgoal reachability penalty. Of these, the former indicates the difference between the currently learned hierarchical policies $\pi_h$ and $\pi_l$ and the optimal hierarchical policies $\pi^*_h$ and $\pi^*_l$. Since $(1+\pi^*_h/\pi_h)$ and $\epsilon^g_{\pi^*_l,\pi_l}$ are decoupled from each other, this inspires us to optimize the high and low hierarchical policies separately to reduce the policies' inconsistency and improve the performance of the policies. More importantly, the core difference from previous work is that the subgoal reachability penalty matters, which requires reduction from both high- and low-level policies, thus we integrate it into the optimization procedures of the two levels.

\section{Experiment}
Our experimental evaluation aims to investigate the following questions: 
1) How does BrHPO's performance on long-term goal-conditioned benchmark tasks compare to that of state-of-the-art counterparts in terms of sample efficiency and asymptotic performance? 
2) How effective is the mutual response mechanism in enhancing subgoal reachability and improving performance?
\begin{figure*}[t]
\centering
\begin{subfigure}[t]{0.14\textwidth}
    \centering
    \includegraphics[width=\textwidth]{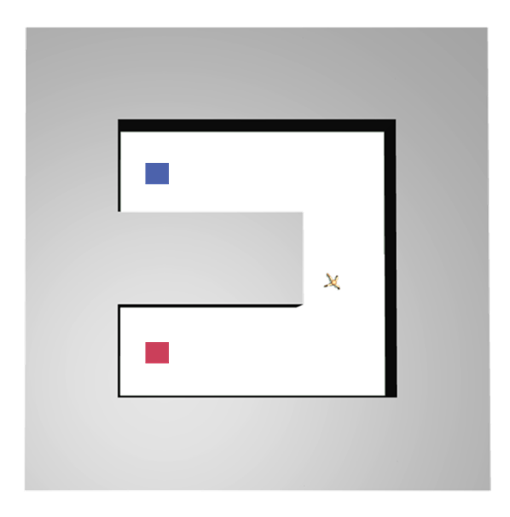}
    \caption{AntMaze}
\end{subfigure}
\begin{subfigure}[t]{0.22\textwidth}
    \centering
    \includegraphics[width=\textwidth]{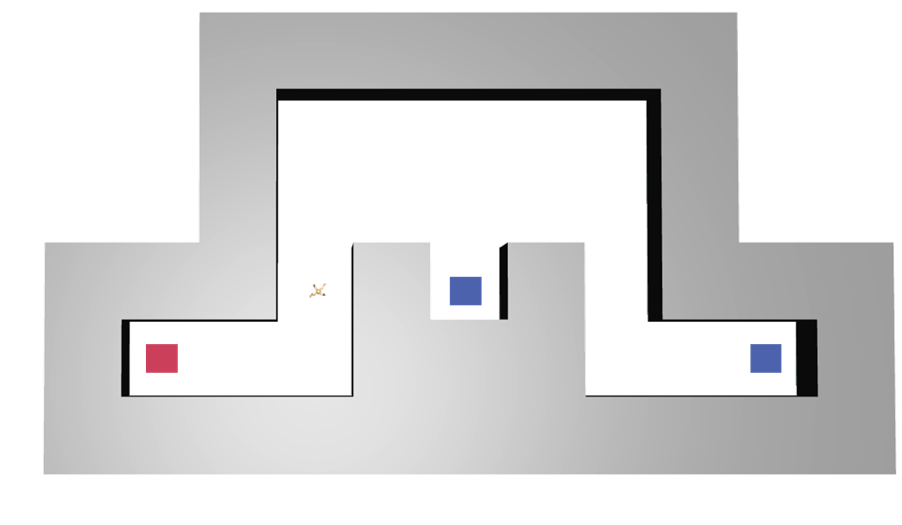}
    \caption{AntBigMaze}
\end{subfigure}
\begin{subfigure}[t]{0.14\textwidth}
    \centering
    \includegraphics[width=\textwidth]{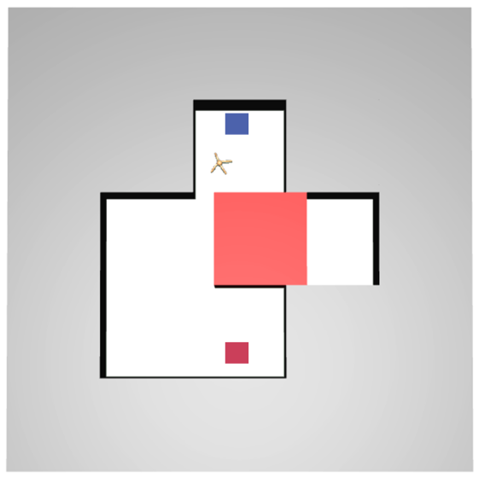}
    \caption{AntPush}
\end{subfigure}
\begin{subfigure}[t]{0.12\textwidth}
    \centering
    \includegraphics[width=\textwidth]{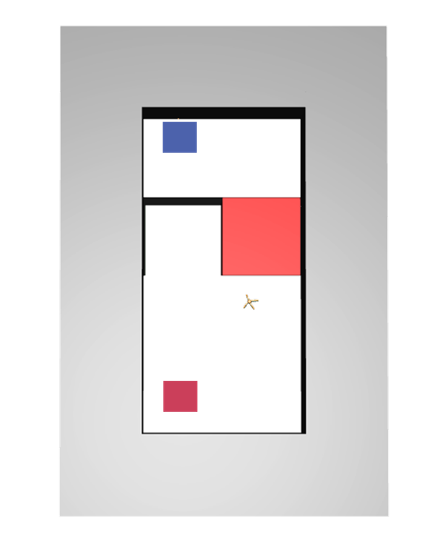}
    \caption{AntFall}
\end{subfigure}
\begin{subfigure}[t]{0.17\textwidth}
    \centering
    \includegraphics[width=\textwidth]{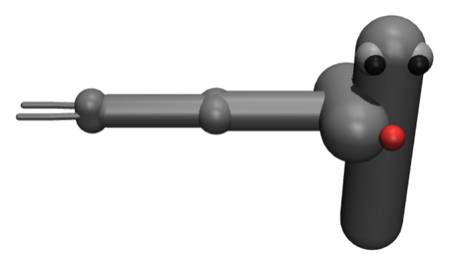}
    \caption{Reacher3D}
\end{subfigure}
\begin{subfigure}[t]{0.17\textwidth}
    \centering
    \includegraphics[width=\textwidth]{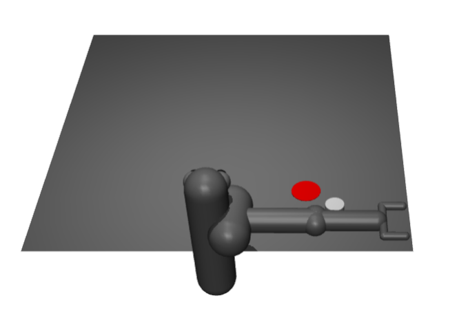}
    \caption{Pusher}
\end{subfigure}
\caption{Environments used in our experiments. In maze tasks, the red square indicates the start point and the blue square represents the target point. In manipulation tasks, a robotic arm aims to make its end-effector and (puck-shaped) grey object reach the target position, which is marked as a red ball, respectively.}
\label{env}
\end{figure*}
\paragraph{Experimental setup}
We evaluate BrHPO on two categories of challenging long-horizon continuous control tasks, which feature both 
\textit{dense} and \textit{sparse} environmental reward, as illustrated in Figure~\ref{env}. 
In the maze navigation environments, the reward is determined by the 
negative $\mathcal{L}_2$ distance between the current state and the target position within the goal space. In the robotics manipulation environments with sparse rewards, 
the reward is set to $0$ when the distance is below a predefined threshold; otherwise, it's set to $-1$. Task success is defined as achieving a final distance to the target 
point of $d\leq5$ for the maze tasks and $d\leq0.25$ for the manipulation tasks. 
To ensure a fair comparison, all agents are initialized at the same position, eliminating extra 
environmental information introduction from random initialization~\citep{lee2022dhrl}. Detailed settings can be found in Appendix \textcolor{mydarkblue}{B}.

\subsection{Comparative evaluation}
We compared BrHPO with the following baselines.
1) \textit{HIRO}~\citep{nachum2018data}: designed an off-policy correction mechanism which required high-level experience to obey the current low-level policy;
2) \textit{HIGL}~\citep{kim2021landmark}: relied on the off-policy correction mechanism and introduced a $k$-step adjacent constraint~\citep{zhang2020generating} and the novelty to discover appropriate subgoals;
3) \textit{RIS}~\citep{chane2021goal}: utilized the hindsight method to generate the least-cost middle points as subgoals, forcing the low-level policy to follow the given subgoals;
4) \textit{CHER}~\citep{kreidieh2019inter}: considered the cooperation of hierarchical policies, and the high-level policy needs to care about the low-level behaviour per step;
5) \textit{SAC}~\citep{haarnoja2018soft}: served as a benchmark of flat off-policy model-free algorithm and was applied as the backbone of BrHPO.
Simply put, HIRO and HIGL focused on low-level domination, and RIS focused on high-level domination. CHER also considers the cooperation of different level policies while it requires step-by-step consideration.

The learning curves of BrHPO and the baselines across all tasks are plotted in Figure~\ref{env_result}.
Overall, the results demonstrate that BrHPO outperforms all baselines both in exploration efficiency and asymptotic performance.
In particular, when dealing with large-scale (AntBigMaze) and partially-observed environments (AntPush and AntFall), BrHPO achieves better exploration and training stability, benefitting from the mutual response mechanism with information sharing and error correction for both levels. 
In contrast, acceptable baselines like HIRO, HIGL and CHER exhibit performance fluctuations and low success rates.
It's worth noting that BrHPO can handle \textit{sparse} reward environments without any reward shaping or hindsight relabeling modifications, indicating that our proposed mechanism can capture serendipitous success and provide intrinsic guidance. Besides, we report the training wall-time in Appendix~\ref{appendix:training_time}, indicating that our method can achieve efficient computational performance, with training times comparable to a flat SAC policy.
Notably, compared to previous approaches that utilize adjacency matrices (HRAC) or graphs to model subgoal reachability (HIGL), our method achieves at least a \textbf{2x} improvement in training efficiency with performance guarantee.
\begin{figure*}[t] 
\centering
\begin{subfigure}[t]{0.32\textwidth}
    \centering
    \includegraphics[width=\textwidth]{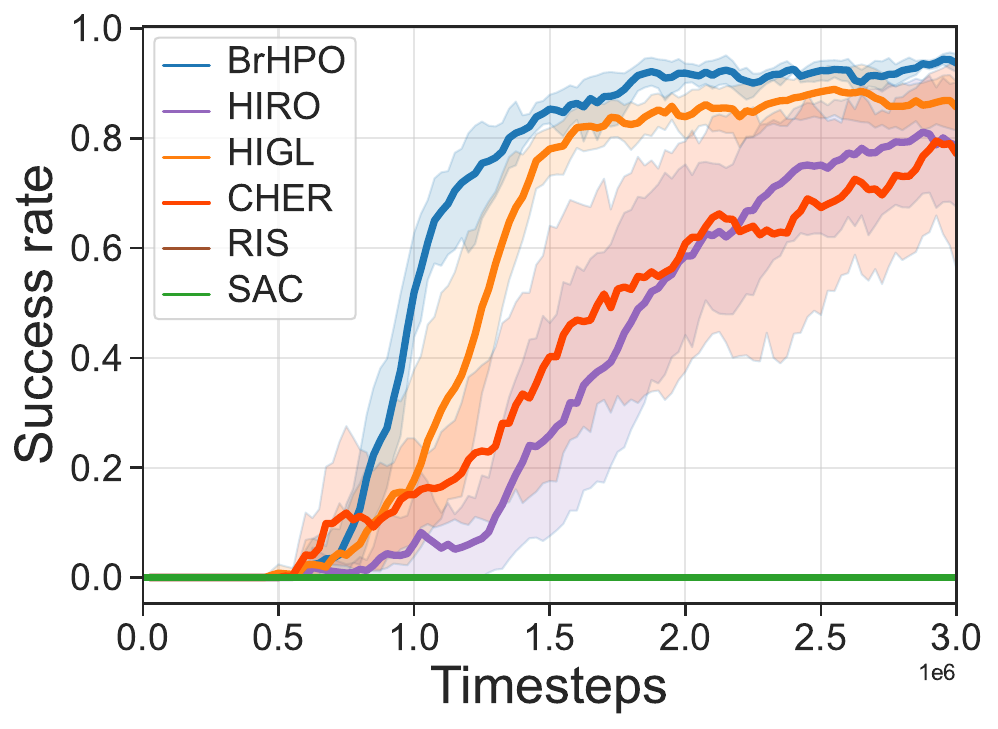}
    \caption{AntMaze}
\end{subfigure}
\begin{subfigure}[t]{0.32\textwidth}
    \centering
    \includegraphics[width=\textwidth]{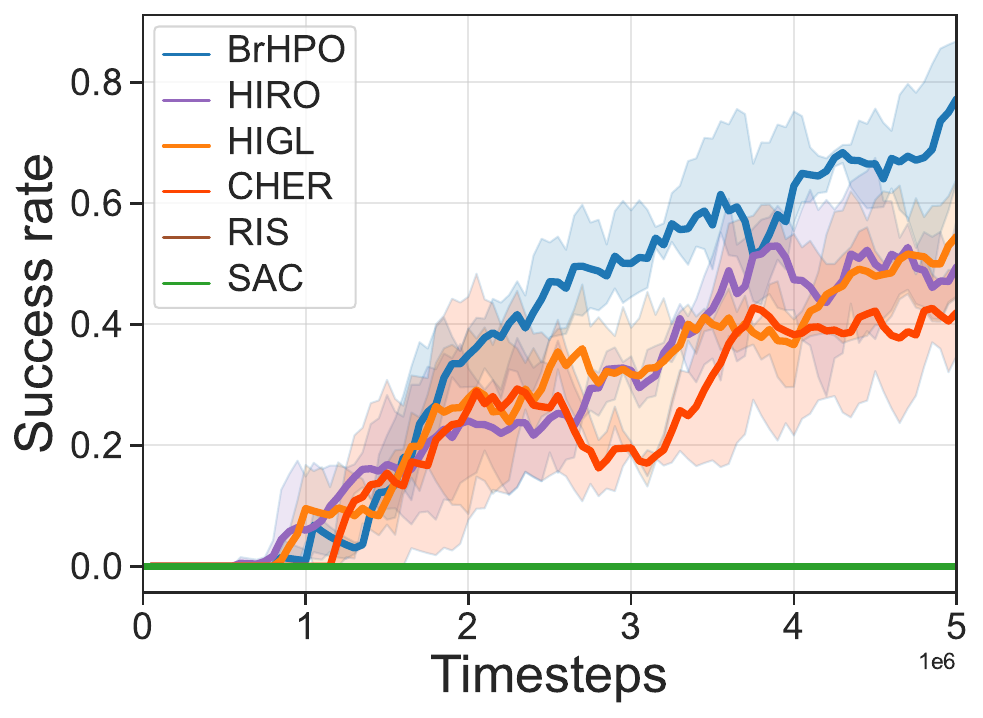}
    \caption{AntBigMaze}
\end{subfigure}
\begin{subfigure}[t]{0.32\textwidth}
    \centering
    \includegraphics[width=\textwidth]{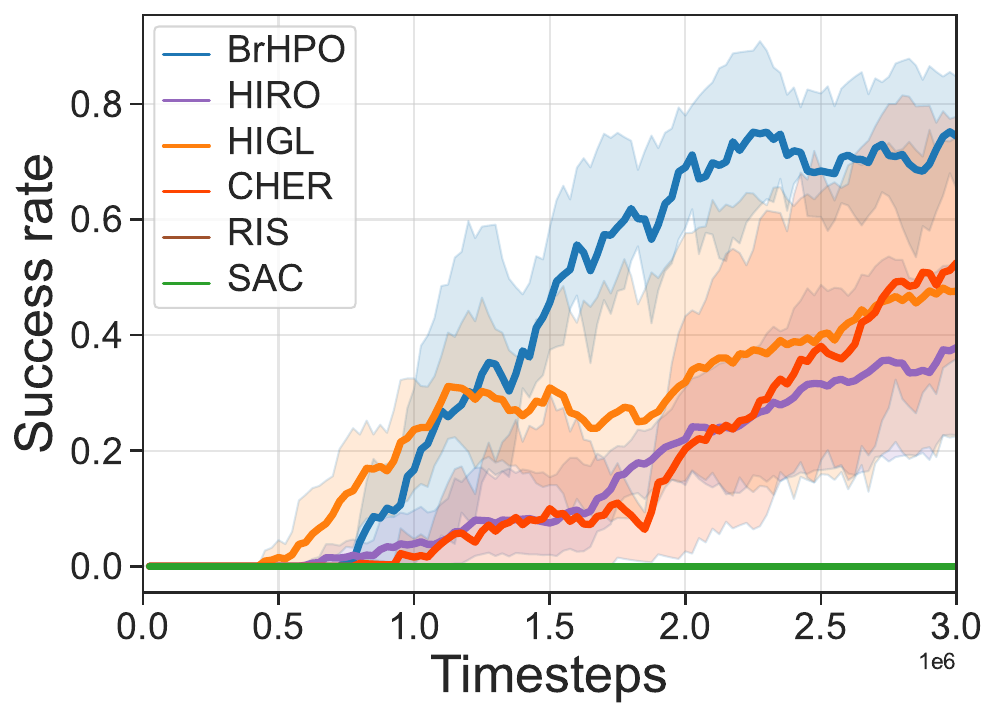}
    \caption{AntPush}
\end{subfigure}
\begin{subfigure}[t]{0.32\textwidth}
    \centering
    \includegraphics[width=\textwidth]{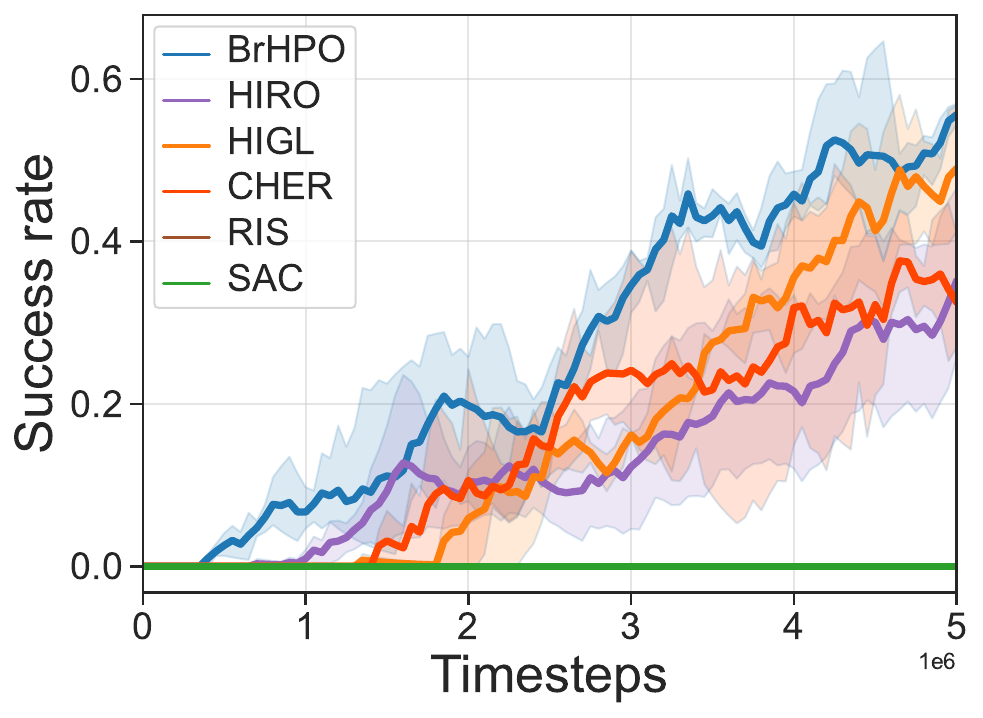}
    \caption{AntFall}
\end{subfigure}
\begin{subfigure}[t]{0.32\textwidth}
    \centering
    \includegraphics[width=\textwidth]{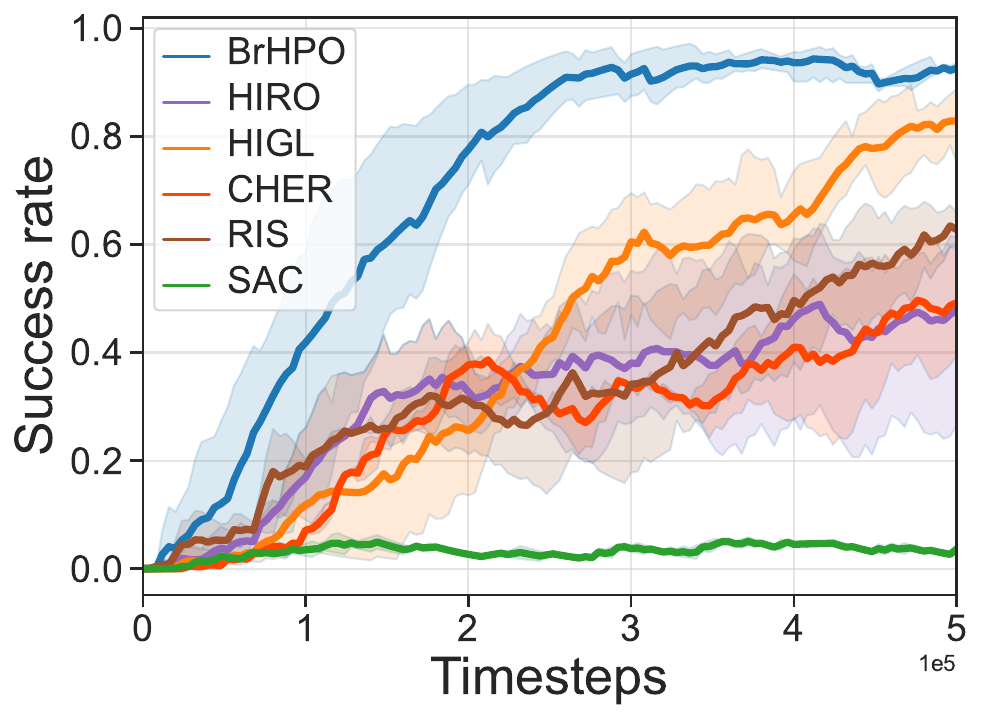}
    \caption{Reacher3D (Sparse)}
\end{subfigure}
\begin{subfigure}[t]{0.32\textwidth}
    \centering
    \includegraphics[width=\textwidth]{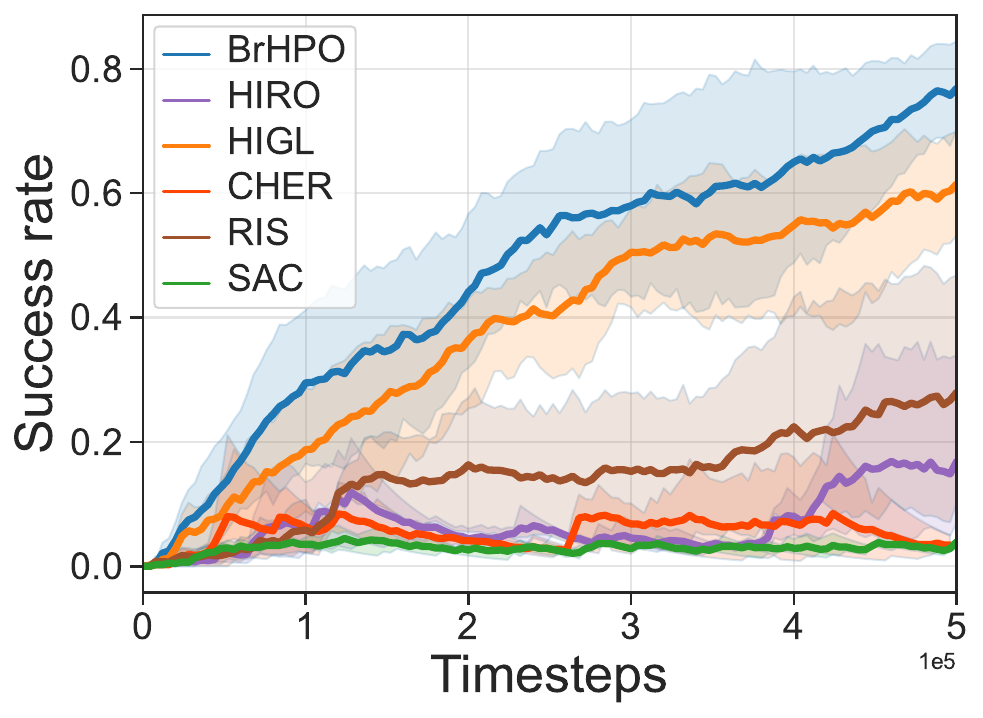}
    \caption{Pusher (Sparse)}
\end{subfigure}
\caption{The average success rate in various continuous control tasks of BrHPO and baselines. The solid lines are the average success rate, while the shades indicate the standard error of the average performance. All algorithms are evaluated with $5$ random seeds.}
\label{env_result}
\end{figure*}
\begin{figure*}[t]
  \centering
  \begin{subfigure}[t]{0.24\textwidth}
    \centering
    \includegraphics[width=\textwidth]{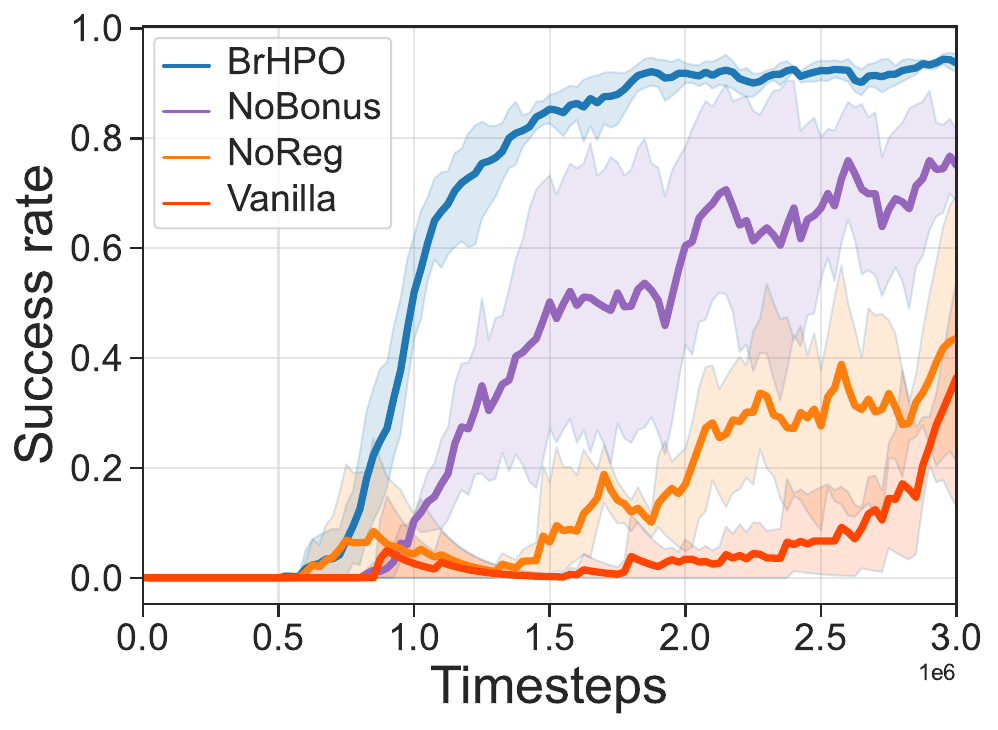}
    \caption{AntMaze}
  \end{subfigure}
  \begin{subfigure}[t]{0.24\textwidth}
    \centering
    \includegraphics[width=\textwidth]{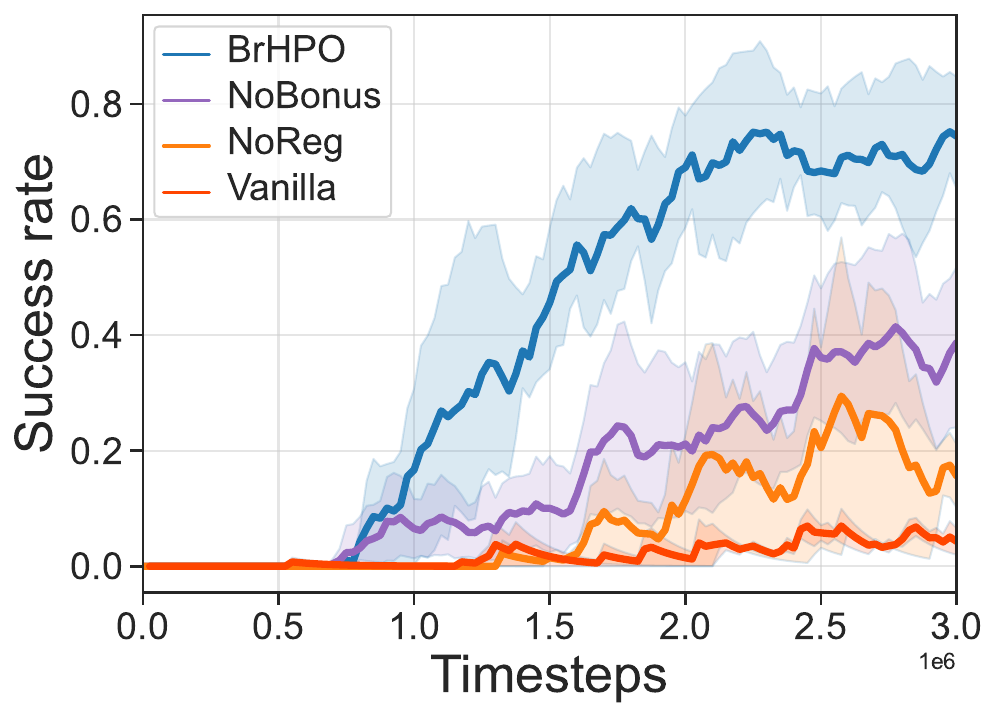}
    \caption{AntPush}
  \end{subfigure}
  \hfill
  \begin{subfigure}[t]{0.49\textwidth}
    \centering
    \includegraphics[width=\textwidth]{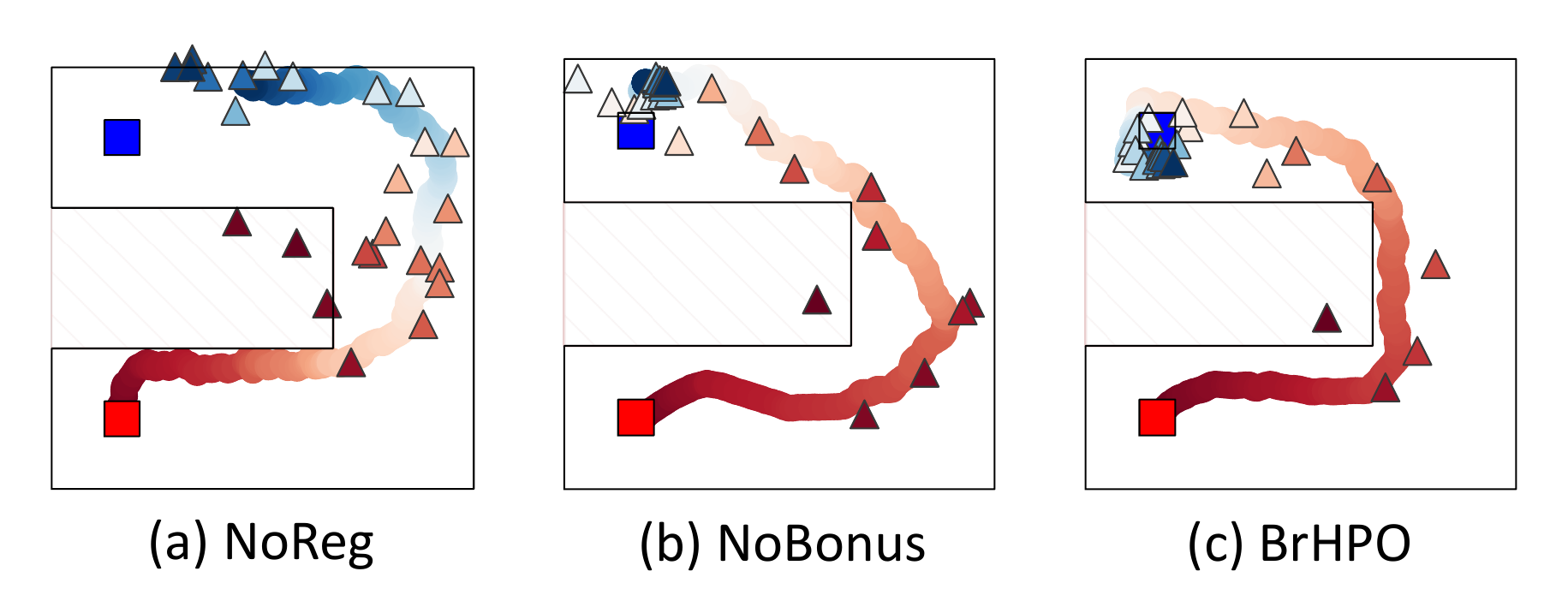}
    \caption{State-subgoal trajectory visualization}
  \end{subfigure}
  \caption{The performance and state-subgoal trajectory visualization from different BrHPO variants.}
  \label{ablation_variants}
\vspace{-0.3cm}
\end{figure*}
\begin{figure*}[t]
  \centering
  \begin{minipage}[t]{.49\textwidth}
    \centering
    \begin{subfigure}[t]{0.49\textwidth}
      \centering
      \includegraphics[width=\textwidth]{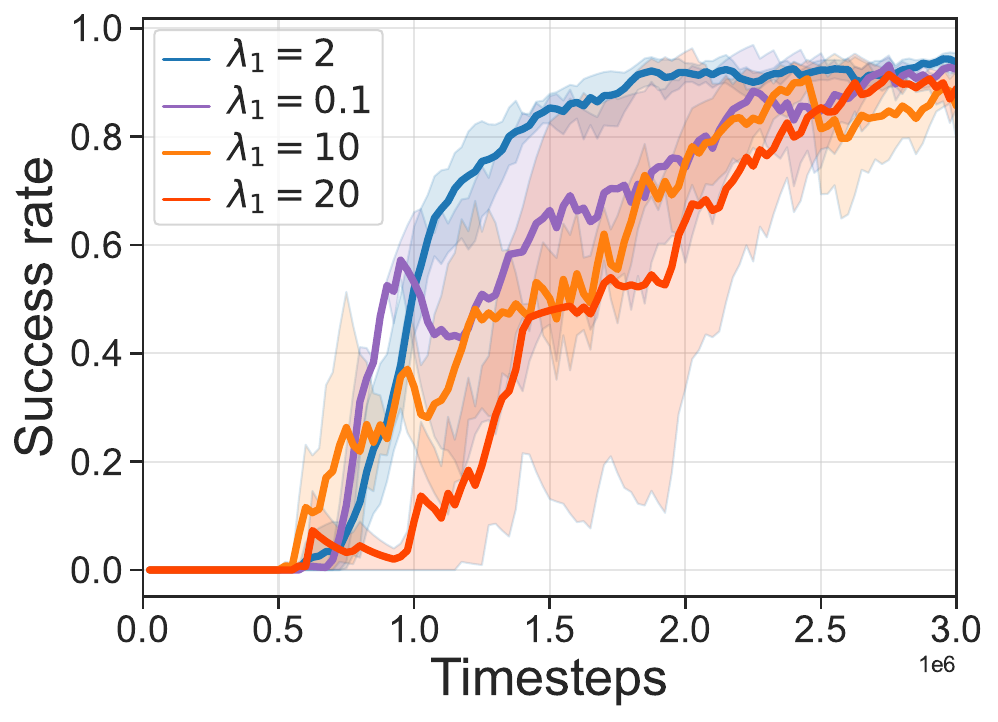}
      \caption{Ablation on $\lambda_1$}
    \end{subfigure}
    \begin{subfigure}[t]{0.49\textwidth}
      \centering
      \includegraphics[width=\textwidth]{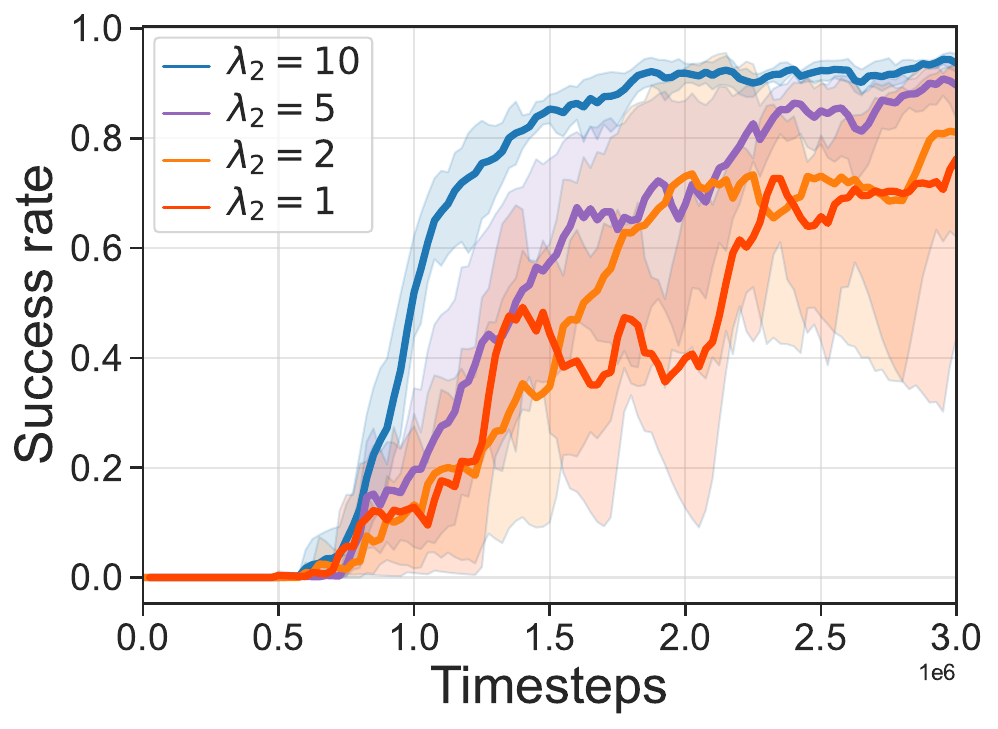}
      \caption{Ablation on $\lambda_2$}
    \end{subfigure}
    \caption{The learning curves with different weight factors $\lambda_1$ and $\lambda_2$ by AntMaze task.}
    \label{ablation_para}
  \end{minipage}
  \hfill
  \begin{minipage}[t]{0.49\textwidth}
    \centering
    \begin{subfigure}[t]{0.49\textwidth}
        \centering
        \includegraphics[width=\textwidth]{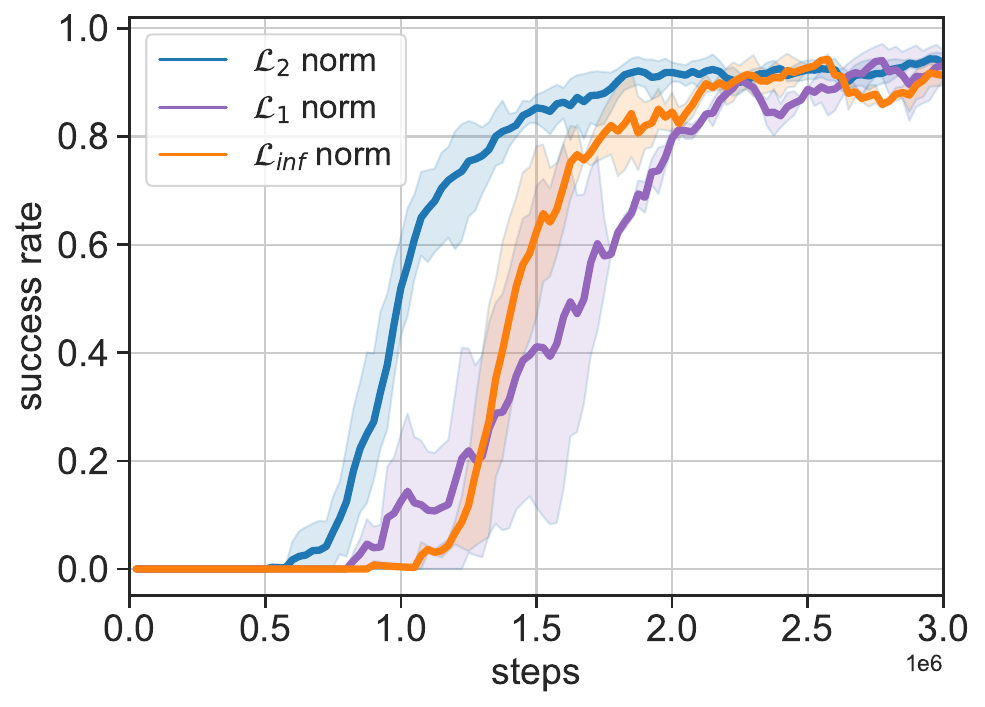}
        \caption{Diff. distance norms}
        \label{different_distance}
    \end{subfigure}
    \begin{subfigure}[t]{0.49\textwidth}
        \centering
        \includegraphics[width=\textwidth]{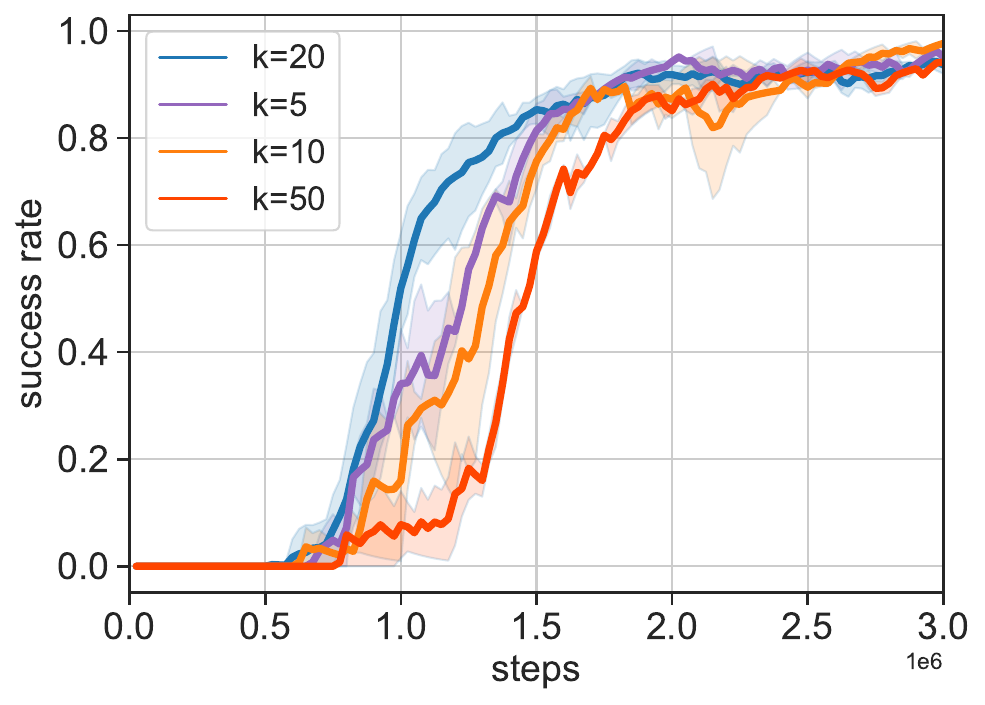}
        \caption{Diff. subtask horizons}
        \label{different_horizon}
    \end{subfigure}
    \caption{The learning curves from different $\mathcal{D}$ and $k$ to verify the robustness of the mechanism.}
\end{minipage}
\end{figure*}

\begin{figure*}[t]
    \begin{minipage}[b]{0.49\textwidth}
        \centering
        \includegraphics[width=\linewidth]{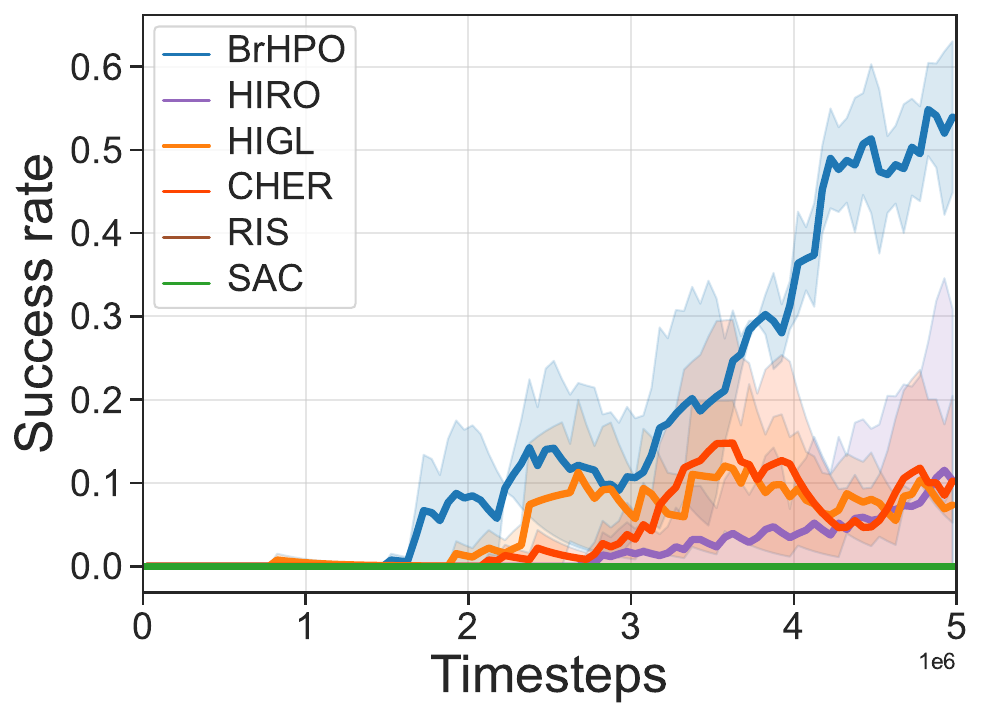}
        \caption{Learning curves of all methods. Mean and std by 4 runs.}
        \label{fig:learning_curves}
    \end{minipage}
    \hfill
    \begin{minipage}[b]{0.49\textwidth}
        \centering
        \subcaptionbox{Visualization of BrHPO.}{\includegraphics[width=\linewidth]{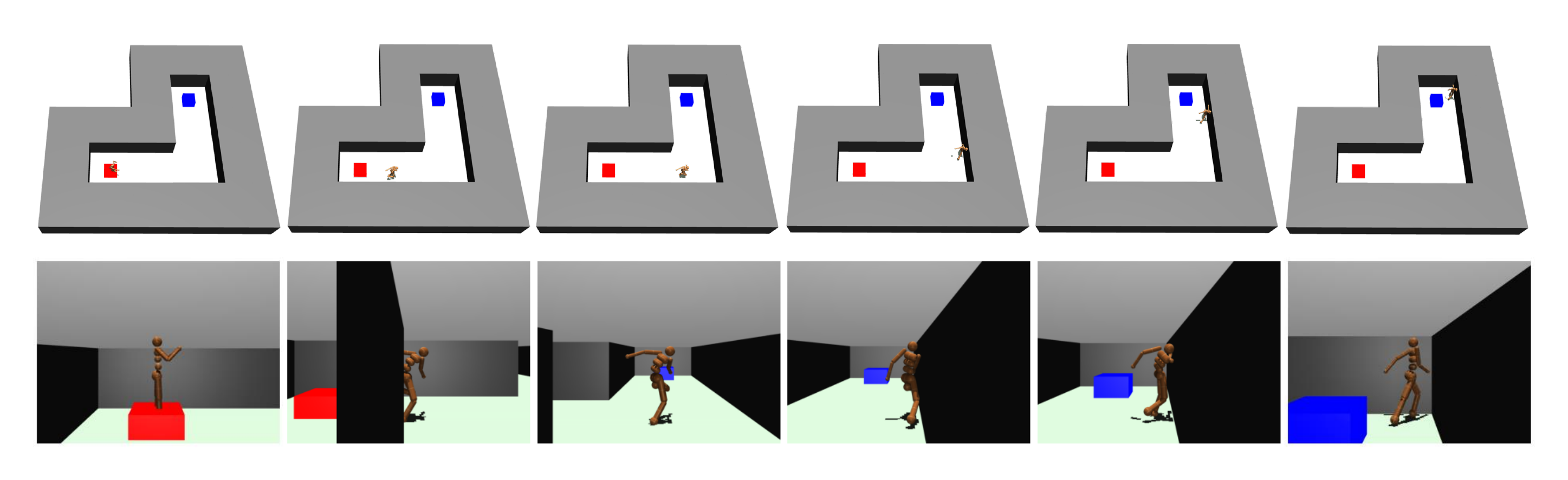}} \\
        \subcaptionbox{Visualization of HIRO.}{\includegraphics[width=\linewidth]{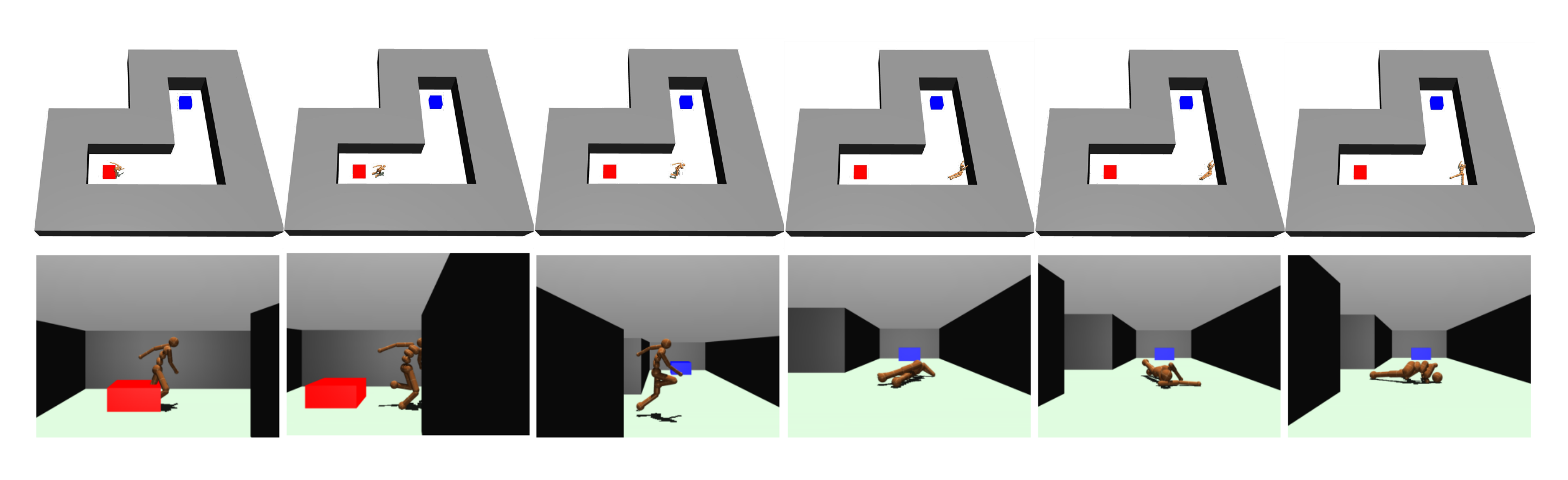}}
        \caption{The behavior comparison between BrHPO and HIRO by HumanoidMaze.}
        \label{fig:humanoidmaze_result}
    \end{minipage}
\end{figure*}

\subsection{Ablation study}\label{Ablation}
Next, we make ablations and modifications to our method to validate the effectiveness and robustness of the mechanism we devised.

\paragraph{Ablation on design choices.}
To investigate the effectiveness of each component, we compared BrHPO with several variants through AntMaze and AntPush tasks.
The BrHPO variants include, 
1) \textit{Vanilla}, which removes the mutual response mechanism in both-level policies, resulting in $\pi_h$ and $\pi_l$ being trained solely by conventional SAC;
2) \textit{NoReg}, which keeps the low-level reward bonus but disables the regularization term in high-level policy training;
3) \textit{NoBonus}, where only the high-level policy concerns subgoal reachability but the low-level reward bonus is removed.

The learning curves and state-subgoal trajectory visualizations from different variants are presented in Figure~\ref{ablation_variants}.
BrHPO outperforms all three variants by a significant margin, highlighting the importance of the mutual response mechanism at both levels. 
Interestingly, the \textit{NoBonus} variant achieves better performance than the \textit{NoReg} variant, suggesting that the subgoal reachability addressed by the high-level policy has a greater impact on overall performance. 
This observation is further supported by the trajectory visualization results.

\paragraph{Hyperparameters.}
We empirically studied the sensitivity of weight factors $\lambda_1$ and $\lambda_2$ in Figure~\ref{ablation_para}.
The results show that $\lambda_1$ and $\lambda_2$ within a certain range are acceptable.
Upon closer analysis, we observed that when $\lambda_1$ is too small, the regularization term in high-level policy optimization has minimal influence. 
Consequently, the high-level policy tends to disregard the performance of the low-level policy during tuning, resembling a high-level dominance scenario. 
Conversely, when $\lambda_1$ is too large, the high-level policy overly prioritizes subgoal reachability, diminishing its exploration capability and resembling a low-level dominance scenario.
These observations validate the effectiveness of the mutual response mechanism in maintaining a balanced interaction between the high- and low-level policies. 
Additionally, the results for $\lambda_2$ suggest that a larger value can generally improve subgoal reachability from the perspective of the low-level policy, leading to performance improvements and enhanced stability.

\paragraph{Robustness of mutual response mechanism.}
We conducted additional experiments on the AntMaze task to verify the robustness of the proposed mechanism.
The computation of subgoal reachability, a key factor in the mutual response mechanism, depends on the choice of the distance measurement $\mathcal{D}$ and the subtask horizon $k$.
To test the distance measurement $\mathcal{D}$, we compared three distance functions: $\mathcal{L}_2$ norm, $\mathcal{L}_\infty$ norm, and $\mathcal{L}_1$ norm. 
Figure~\ref{different_distance} shows that our method performs well regardless of the distance function used, highlighting the adaptability of the proposed mechanism.
Additionally, we varied the subtask horizon by setting $k=5,10,20,50$ (Figure~\ref{different_horizon}).

Surprisingly, we achieved success rates of around 0.9 with different subtask horizons, indicating that the performance is robust to variations in the subtask horizon, only with a slight effect on the convergence speed during training. 
This flexibility of BrHPO in decoupling the high- and low-level horizons without the need for extra graphs, as required in DHRL~\citep{lee2022dhrl}, is noteworthy. More ablations by Reacher3D task are provided in Figure~\ref{Reacher3D_ablation_variants} of Apppendix~\ref{sec:additional_results}.

In addition to evaluating parametric robustness, we subjected BrHPO to testing in stochastic environments to further evaluate its robustness. As depicted in Figure~\ref{stochastic_ablation_variants} of Apppendix~\ref{sec:additional_results}, we introduce varying levels of Gaussian noise into the state space. The results demonstrate our BrHPO can effectively mitigate the impact of noise and ensure consistent final performance.

\paragraph{Mutual response mechanism in complex tasks.}
Except for the main results, we consider a more complex robot in a maze, HumanoidMaze, to further evaluate the mutual response mechanism. In this task, the simulated humanoid, where the state space contains $\mathbf{274}$ dimensions and the action space is $\mathbf{17}$, needs to maintain body balance while being guided by the subgoal from the high-level policy. Consequently, the low-level policy necessitates extensive training to facilitate the humanoid's ability to learn how to walk proficiently. This training process requires the high-level policy to exhibit ``patience'', gradually adjusting the subgoals to guide the humanoid's progress effectively. Figure~\ref{fig:learning_curves} demonstrates the performance comparison, which showcases the superior advantage of BrHPO over HIRO. We additionally visualize the trajectory in Figure~\ref{fig:humanoidmaze_result}. We find that, our mutual response mechanism can encourage cooperation between the high- and the low-level policies, while the erroneous guidance from HIRO makes it difficult for humanoid to maintain balance and easily fall, thus failing the task.

\section{Related Works}
Hierarchical Reinforcement Learning (HRL) methods have emerged as promising solutions for addressing long-horizon complex tasks, primarily due to the synergistic collaboration between high-level task division and low-level exploration~\citep{jong2008utility,haarnoja2018latent,nachum2019does,pateria2021hierarchical,eppe2022intelligent}. 
Generally, HRL methods can be broadly categorized into two groups, option-based HRL~\citep{sutton1999between,precup1998theoretical,zhanghierarchical,mannor2004dynamic} and subgoal-based HRL~\citep{dayan1992feudal,nachum2019does,campos2020explore,li2021learning,islam2022discrete}, that highlights the scope of guidance provided by the high-level policy. 
The first avenue involves the use of options to model the policy-switching mechanism in long-term tasks, which provides guidance to the low-level policy on when to terminate the current subtask and transition to a new one~\citep{machado2017laplacian,zhang2019dac}. 
In contrast, the subgoal-based HRL avenue~\citep{vezhnevets2017feudal,nachum2018near,gurtler2021hierarchical,czechowski2021subgoal,li2021active} focuses on generating subgoals in fixed horizon subtasks rather than terminal signals, and our work falls under this category. Notably, subgoal-based HRL approaches prioritize subgoal reachability as a means of achieving high performance~\citep{stein2018learning,paul2019learning,li2020learning,czechowski2021subgoal,pateria2021end}. 

Various methods have been proposed to enhance subgoal reachability, from either the high-level or low-level perspectives. When the low level is considered to be dominant, several works have proposed relabeling or correcting subgoals based on the exploration capacity of the low-level policy. Examples include off-policy correction in HIRO~\citep{nachum2018data} and hindsight relabeling in HER~\citep{andrychowicz2017hindsight}, RIS~\citep{chane2021goal} and HAC~\citep{levylearning}. On the other hand, when the high-level dominates, subgoals are solved from given prior experience or knowledge, and the low-level policy is trained merely to track the given subgoals~\citep{savinovsemi,huang2019mapping,eysenbach2019search,jurgenson2020sub}. In contrast to the listed prior works, BrHPO proposes a mutual response mechanism for ensuring bidirectional reachability.
SoRB~\citep{eysenbach2019search}, for instance, constructs an environmental graph from the given replay buffer for high-level planning and uses the waypoints as subgoals. SGT~\citep{jurgenson2020sub} adopts a divide-and-conquer mechanism to search intermediate subgoals from given trajectories. 

Meanwhile, our method relates to previous research that encourages cooperation between the high-level policy and the low-level one, where they explored various techniques for modelling subgoal reachability, including $k$-step adjacency matrix~\citep{ferns2004metrics,castro2020scalable,zhang2020generating} or state-subgoal graph~\citep{zhang2018composable,kim2021landmark, lee2022dhrl}. However, these methods can be computationally intensive and conservative. Our proposed method provides a more computationally efficient and flexible approach to gain subgoal reachability. By avoiding an explicit representation of the state-subgoal adjacency, our method can be more easily deployed and applied to a variety of different environments.

\section{Conclusion}
In this work, we identify that bilateral information sharing and error correction have been long neglected in previous HRL works. This will potentially cause local exploration and unattainable subgoal generation, which hinders overall performance and sample efficiency. To address this issue, we delve into the mutual response of hierarchical policies, both theoretically and empirically, revealing the crucial role of the mutual response mechanism. Based on these findings, we proposed the Bidirectional-reachable Hierarchical Policy Optimization (BrHPO) algorithm. BrHPO not only matches the best HRL algorithms in asymptotic performance, but it also shines in low computational load.
Although BrHPO offers many advantages, a main challenge is to design an appropriate low-level reward to compute the subgoal reachability, thus limiting the application in sparse low-level reward settings~\citep{lee2022dhrl}. Future work that merits investigation are integrating up-to-date reachability measurement for bidirectional subgoal reachability and policy optimization backbone to develop a strong HRL algorithm.

\subsubsection*{Acknowledgments}
\label{sec:ack}
This work was done at Tsinghua University and supported
by the National Science and Technology Major Project (2021ZD0113801). We appreciate the reviewers' generous help to improve our paper.

\bibliography{main}
\bibliographystyle{rlc}

\clearpage
\appendix
\section{Theoretical Analysis}
\subsection{Omitted Proofs}\label{appendix:omitter_proof}
\begin{theorem}[Sub-optimal performance difference bound of HRL]
The performance difference bound $C$ between the induced optimal hierarchical policies $\Pi^*$ and the learned one $\Pi$ can be
\begin{align}
C(\pi_h,\pi_l)=
\frac{2r_{max}}{(1-\gamma)^2}\Bigg
[\underbrace{(1+\gamma)\mathbb{E}_{g\sim\pi_h}\left(1+\frac{\pi^*_h}{\pi_h}\right)\epsilon^g_{\pi^*_l,\pi_l}}_{\text{(\romannumeral1) hierarchical policies' inconsistency}}+\underbrace{2\left(\mathcal{R}^{\pi_h,\pi_l}_{max} + 2\gamma^k\right)}_{\text{(\romannumeral2) subgoal reachability penalty}}\Bigg],
\end{align}
where $\epsilon^g_{\pi^*_l,\pi_l}$ is the distribution shift between $\pi^*_l$ and $\pi_l$, and $\mathcal{R}^{\pi_h,\pi_l}_{max}$ is the maximum subgoal reachability penalty from the learned one $\Pi$, both of which are formulated as,
\begin{equation*}
\epsilon^g_{\pi^*_l,\pi_l}=\max_{s\in\mathcal{S},g\sim\pi_h}{\rm D}_{TV}\left(\pi^*_l(\cdot|s,g)\Vert\pi_l(\cdot|s,g)\right) \quad \text{and} \quad \mathcal{R}^{\pi_h,\pi_l}_{max}=\max_{i\in\mathbb{N}}\mathcal{R}^{\pi_h,\pi_l}_i.
\end{equation*}
\end{theorem}

\textit{Summary of proof.} We first divide the bound into three parts, $V^{\Pi^*}(s)-V^{\Pi}(s)=\underbrace{V^{\pi^*_h,\pi^*_l}(s_0)-V^{\pi^*_h,\pi_l}(s_0)}_{L_1}+\underbrace{V^{\pi^*_h,\pi_l}(s_0)-V^{\pi_h,\pi^*_l}(s_0)}_{L_2}+ \underbrace{V^{\pi_h,\pi^*_l}(s_0)-V^{\pi_h,\pi_l}(s_0)}_{L_3}$. Then, we find the similarity of $L_1$ and $L_3$, both of which denote that under the same high-level policy ($\pi^*_h$ in $L_1$ while $\pi_h$ in $L_3$). By Performance Difference Lemma~\citep{kakade2002approximately}, we have 
\begin{equation}\label{eq:l1_l3_plus}
L_1 + L_3\leq\frac{2r_{max}}{(1-\gamma)^2}\mathbb{E}_{g\sim\pi_h}\left(1+\frac{\pi^*_h}{\pi_h}\right)\epsilon^g_{\pi^*_l,\pi_l}.
\end{equation}

For $L_2$, we follow~\citet{zhang2022adjacency} and substitute equation~(\ref{reachability_function}). Then we have
\begin{equation}\label{eq:l2}
L_2\leq\frac{r_{max}}{(1-\gamma)^2}(\mathcal{R}^{\pi_h,\pi_l}_{max}+2\gamma^k).
\end{equation}
Thus, we take the results of Equations~(\ref{eq:l1_l3_plus}) and (\ref{eq:l2}) and achieve the final bound.

\begin{proof}
To derive the performance difference bound between $\Pi^*$ and $\Pi$, we first divide the bound into three terms,
\begin{align}
V^{\Pi^*}(s_0)-V^{\Pi}(s_0)&=V^{\pi^*_h,\pi^*_l}(s_0)-V^{\pi_h,\pi_l}(s_0)\nonumber\\
&=\underbrace{V^{\pi^*_h,\pi^*_l}(s_0)-V^{\pi^*_h,\pi_l}(s_0)}_{L_1}+\underbrace{V^{\pi^*_h,\pi_l}(s_0)-V^{\pi_h,\pi^*_l}(s_0)}_{L_2}+ \underbrace{V^{\pi_h,\pi^*_l}(s_0)-V^{\pi_h,\pi_l}(s_0)}_{L_3}.
\end{align}

Then, our proof can be obtained to by tackling $L_1$, $L_2$ and $L_3$, respectively.

\begin{itemize}
  \item \textbf{Derivation of $L_1$}
\end{itemize}

By adding and subtracting the same term in $L_1$, we obtain
\begin{align}
L_1&=V^{\pi^*_h,\pi^*_l}(s_0)-\left[\widetilde{V}^{\pi^*_h,\pi_l^*}_0(s_0)+\gamma^k\mathbb{E}_{g\sim\pi^*_h,s\sim \mathbb{P}^{\pi_l^*,g}_k(\cdot|s_0)}V^{\pi^*_h,\pi_l}(s_k)\right]\nonumber \\
&\quad\quad+\left[\widetilde{V}^{\pi^*_h,\pi_l^*}_0(s_0)+\gamma^k\mathbb{E}_{g\sim\pi^*_h,s\sim \mathbb{P}^{\pi_l^*,g}_k(\cdot|s_0)}V^{\pi^*_h,\pi_l}(s_k)\right]-V^{\pi^*_h,\pi_l}(s_0)\nonumber \\
&=\left[\widetilde{V}^{\pi^*_h,\pi_l^*}_0(s_0)+\gamma^k\mathbb{E}_{g\sim\pi^*_h,s\sim \mathbb{P}^{\pi_l^*,g}_k(\cdot|s_0)}V^{\pi^*_h,\pi^*_l}(s_k)\right] \quad \leftarrow \quad \text{\textcolor{myblue}{By Lemma~\ref{Bellman_Backup_HRL}}}\nonumber\\
&\quad\quad-\left[\widetilde{V}^{\pi^*_h,\pi_l^*}_0(s_0)+\gamma^k\mathbb{E}_{g\sim\pi^*_h,s\sim \mathbb{P}^{\pi_l^*,g}_k(\cdot|s_0)}V^{\pi^*_h,\pi_l}(s_k)\right]\nonumber\\
&\quad\quad\quad\quad+\left[\widetilde{V}^{\pi^*_h,\pi_l^*}_0(s_0)+\gamma^k\mathbb{E}_{g\sim\pi^*_h,s\sim \mathbb{P}^{\pi_l^*,g}_k(\cdot|s_0)}V^{\pi^*_h,\pi_l}(s_k)\right]\nonumber\\
&\quad\quad\quad\quad\quad\quad-\left[\widetilde{V}^{\pi^*_h,\pi_l}_0(s_0)+\gamma^k\mathbb{E}_{g\sim\pi^*_h,s\sim \mathbb{P}^{\pi_l,g}_k(\cdot|s_0)}V^{\pi^*_h,\pi_l}(s_k)\right]\nonumber \\
&=\underbrace{\gamma^k\mathbb{E}_{g\sim\pi^*_h,s\sim \mathbb{P}^{\pi_l^*,g}_k(\cdot|s_0)}\left[V^{\pi^*_h,\pi^*_l}(s_k)-V^{\pi^*_h,\pi_l}(s_k)\right]}_{\text{part a}}+\underbrace{\left[\widetilde{V}^{\pi^*_h,\pi_l^*}_0(s_0)-\widetilde{V}^{\pi^*_h,\pi_l}_0(s_0)\right]}_{\text{part b}}\nonumber \\
&\quad+\underbrace{\gamma^k\left[\mathbb{E}_{g\sim\pi^*_h,s\sim \mathbb{P}^{\pi_l^*,g}_k(\cdot|s_0)}V^{\pi^*_h,\pi_l}(s_k)-\mathbb{E}_{g\sim\pi^*_h,s\sim \mathbb{P}^{\pi_l,g}_k(\cdot|s_0)}V^{\pi^*_h,\pi_l}(s_k)\right]}_{\text{part c}}.
\end{align}

Then, we can deal with the three parts one by one to obtain the derivation of $L_1$. Note that, part b represents the performance discrepancy in the first subtask, caused by different low-level policies $\pi^*_l$ and $\pi_l$. Thus, consider the policy shift of the low-level policies, we suppose
\begin{equation}
\epsilon^g_{\pi^*_l,\pi_l}=\max_{s\in\mathcal{S},g\sim\pi_h}{\rm D}_{TV}\left(\pi^*_l(\cdot|s,g)\Vert\pi_l(\cdot|s,g)\right).
\end{equation}
 
Then, recall $r_{max}$ to be the maximum environmental reward, i.e., $r\leq r_{max}$, we have
\begin{align}
\text{part b}&=\widetilde{V}^{\pi^*_h,\pi_l^*}_0(s_0)-\widetilde{V}^{\pi^*_h,\pi_l}_0(s_0)\nonumber\\
&=\sum_{j=0}^{k-1}\mathbb{E}_{g_{k}\sim\pi^*_h,s,a\sim \mathbb{P}^{\pi^*_l,g}_{j}(\cdot,\cdot|s_0)}\left[\gamma^jr(s_j,a_j,\hat{g})\right] - \sum_{j=0}^{k-1}\mathbb{E}_{g_{k}\sim\pi^*_h,s,a\sim \mathbb{P}^{\pi_l,g}_{j}(\cdot,\cdot|s_0)}\left[\gamma^jr(s_j,a_j,\hat{g})\right]\nonumber\\
&\leq\sum_{j=0}^{k-1}\mathbb{E}_{g_{k}\sim\pi^*_h}2\left[\gamma^jr(s_j,a_j,\hat{g})\right]{\rm D}_{TV}\left(\mathbb{P}^{\pi^*_l,g}_{j}(\cdot,\cdot|s_0)\Big\Vert\mathbb{P}^{\pi_l,g}_{j}(\cdot,\cdot|s_0)\right) \nonumber\\
&\leq 2r_{max}\sum_{j=0}^{k-1}\mathbb{E}_{g_{k}\sim\pi^*_h}\gamma^j j \epsilon^g_{\pi^*_l,\pi_l}. \quad \leftarrow \quad \text{\textcolor{myblue}{By Lemma~\ref{TVD_distance}}}
\end{align}

For part c, note that the joint value function can be bounded as $V^{\pi_h,\pi_l}(s_0)\leq r_{max}/(1-\gamma)$. We can apply Lemma~\ref{TVD_distance} to bound the discrepancy of the low-level policies, and have
\begin{align}
\text{part c}&=\gamma^k\left[\mathbb{E}_{g\sim\pi^*_h,s\sim \mathbb{P}^{\pi_l^*,g}_k(\cdot|s_0)}V^{\pi^*_h,\pi_l}(s_k)-\mathbb{E}_{g\sim\pi^*_h,s\sim \mathbb{P}^{\pi_l,g}_k(\cdot|s_0)}V^{\pi^*_h,\pi_l}(s_k)\right]\nonumber\\
&=\gamma^k\int_{g\in\mathcal{G}}\int_{s\in\mathcal{S}}\pi^*_h(g|s_k)\left(\mathbb{P}^{\pi_l^*,g}_k(s|s_0)-\mathbb{P}^{\pi_l,g}_k(s|s_0)\right)V^{\pi^*_h,\pi_l}(s)dsdg\nonumber\\
&\leq \frac{2\gamma^{k}r_{max}}{1-\gamma}\mathbb{E}_{g\sim\pi_h^*}\left[{\rm D}_{TV}\left(\mathbb{P}^{\pi_l^*,g}_k(\cdot|s_0)\Vert \mathbb{P}^{\pi_l,g}_k(\cdot|s_0)\right)\right]\nonumber\\
&\leq\frac{2\gamma^{k}r_{max}}{1-\gamma}\mathbb{E}_{g\sim\pi_h^*}k\epsilon^g_{\pi^*_l,\pi_l},
\end{align}

At last, for part a, we can apply the same recursion every $k$ step,
\begin{align}
\text{part a}&=\gamma^k\mathbb{E}_{g\sim\pi^*_h,s\sim \mathbb{P}^{\pi_l^*,g}_k(\cdot|s_0)}\left[V^{\pi^*_h,\pi^*_l}(s_k)-V^{\pi^*_h,\pi_l}(s_k)\right]\nonumber \\
&\leq \gamma^{2k}\mathbb{E}_{g\sim\pi^*_h,s\sim \mathbb{P}^{\pi_l^*,g}_{2k}(\cdot|s_0)}\left[V^{\pi^*_h,\pi^*_l}(s_{2k})-V^{\pi^*_h,\pi_l}(s_{2k})\right]\nonumber \\
&\quad + 2r_{max}\sum_{j=k}^{2k-1}\mathbb{E}_{g\sim\pi^*_h}\gamma^j j \epsilon^g_{\pi^*_l,\pi_l} + \frac{2\gamma^{2k}r_{max}}{1-\gamma}\mathbb{E}_{g\sim\pi_h^*}2k\epsilon^g_{\pi^*_l,\pi_l}.
\end{align}

Now, with the derivation of part a, part b and part c, we can combine these and repeat the recursion step for infinitely many times
\begin{align}
L_1 &= \text{part a} + \text{part b} + \text{part c}\nonumber \\
&\leq 2r_{max}\sum_{j=0}^{k-1}\mathbb{E}_{g_{k}\sim\pi^*_h}\gamma^j j \epsilon^g_{\pi^*_l,\pi_l} + \frac{2\gamma^{k}r_{max}}{1-\gamma}\mathbb{E}_{g\sim\pi_h^*}k\epsilon^g_{\pi^*_l,\pi_l} \nonumber\\
&\quad + 2r_{max}\sum_{j=k}^{2k-1}\mathbb{E}_{g_{2k}\sim\pi^*_h}\gamma^j j \epsilon^g_{\pi^*_l,\pi_l} + \frac{2\gamma^{2k}r_{max}}{1-\gamma}\mathbb{E}_{g\sim\pi_h^*}2k\epsilon^g_{\pi^*_l,\pi_l} \nonumber \\
&\quad \quad + \gamma^{2k}\mathbb{E}_{g\sim\pi^*_h,s\sim \mathbb{P}^{\pi_l^*,g}_{2k}(\cdot|s_0)}\left[V^{\pi^*_h,\pi^*_l}(s_{2k})-V^{\pi^*_h,\pi_l}(s_{2k})\right]\nonumber\\
&\quad \ \ \vdots \nonumber\\
&\leq 2r_{max}\sum_{i=0}^{\infty}\sum_{j=0}^{k-1}\mathbb{E}_{g\sim\pi^*_h}\gamma^{(ik+j)} (ik+j) \epsilon^g_{\pi^*_l,\pi_l} + \frac{\gamma^{(i+1)k}}{1-\gamma}\mathbb{E}_{g\sim\pi_h^*}(i+1)k\epsilon^g_{\pi^*_l,\pi_l}\nonumber\\
&\leq 2r_{max}\frac{1+\gamma}{(1-\gamma)^2}\mathbb{E}_{g\sim\pi^*_h}\epsilon^g_{\pi^*_l,\pi_l}.
\end{align}
Thus, we complete the derivation of $L_1$.

\begin{itemize}
  \item \textbf{Derivation of $L_3$}
\end{itemize}

Compared with $L_1$, the term $L_3$ replaces the high-level policy from $\pi^*_h$ to $\pi_h$. Thus, we directly can get $L_3$ from the results of $L_1$ as
\begin{align}
L_3 \leq2r_{max}\frac{1+\gamma}{(1-\gamma)^2}\mathbb{E}_{g\sim\pi_h}\epsilon^g_{\pi^*_l,\pi_l}.
\end{align}

\begin{itemize}
  \item \textbf{Derivation of $L_2$}
\end{itemize}

Similar to the derivation of $L_1$, by adding and subtracting the same term in $L_2$, we have
\begin{align}
L_2&=V^{\pi^*_h,\pi_l}(s_0)-\left[\widetilde{V}^{\pi^*_h,\pi_l}_0(s_0)+\gamma^k\mathbb{E}_{g\sim\pi^*_h,s\sim \mathbb{P}^{\pi_l,g}_{k}(\cdot|s_0)}V^{\pi_h,\pi^*_l}(s_k)\right]\nonumber \\
&\quad\quad+\left[\widetilde{V}^{\pi^*_h,\pi_l}_0(s_0)+\gamma^k\mathbb{E}_{g\sim\pi^*_h,s\sim \mathbb{P}^{\pi_l,g}_{k}(\cdot|s_0)}V^{\pi_h,\pi^*_l}(s_k)\right]-V^{\pi_h,\pi^*_l}(s_0)\nonumber\\
&=\left[\widetilde{V}^{\pi^*_h,\pi_l}_0(s_0)+\gamma^k\mathbb{E}_{g\sim\pi^*_h,s\sim \mathbb{P}^{\pi_l,g}_{k}(\cdot|s_0)}V^{\pi^*_h,\pi_l}(s_k)\right]\nonumber\\
&\quad\quad-\left[\widetilde{V}^{\pi^*_h,\pi_l}_0(s_0)+\gamma^k
\mathbb{E}_{g\sim\pi^*_h,s\sim \mathbb{P}^{\pi_l,g}_{k}(\cdot|s_0)}V^{\pi_h,\pi^*_l}(s_k)\right]\nonumber\\
&\quad\quad\quad\quad+\left[\widetilde{V}^{\pi^*_h,\pi_l}_0(s_0)+\gamma^k\mathbb{E}_{g\sim\pi^*_h,s\sim \mathbb{P}^{\pi_l,g}_{k}(\cdot|s_0)}V^{\pi_h,\pi^*_l}(s_k)\right]\nonumber\\
&\quad\quad\quad\quad\quad\quad-\left[\widetilde{V}^{\pi_h,\pi^*_l}_0(s_0)+\gamma^k\mathbb{E}_{g\sim\pi_h,s\sim \mathbb{P}^{\pi^*_l,g}_{k}(\cdot|s_0)}V^{\pi_h,\pi^*_l}(s_k)\right]\nonumber\\
&=\underbrace{\gamma^k\mathbb{E}_{g\sim\pi^*_h,s\sim \mathbb{P}^{\pi_l,g}_{k}(\cdot|s_0)}\left[V^{\pi^*_h,\pi_l}(s_k)-V^{\pi_h,\pi^*_l}(s_k)\right]}_{\text{part d}}+\underbrace{\left[\widetilde{V}^{\pi^*_h,\pi_l}_0(s_0)-\widetilde{V}^{\pi_h,\pi^*_l}_0(s_0)\right]}_{\text{part e}}\nonumber\\
&\quad\quad + \underbrace{\gamma^k\left[\mathbb{E}_{g\sim\pi^*_h,s\sim \mathbb{P}^{\pi_l,g}_{k}(\cdot|s_0)}V^{\pi_h,\pi^*_l}(s_k)-\mathbb{E}_{g\sim\pi_h,s\sim \mathbb{P}^{\pi^*_l,g}_{k}(\cdot|s_0)}V^{\pi_h,\pi^*_l}(s_k)\right]}_{\text{part f}}.
\end{align}

According to Assumption~\ref{reward_assumption}, we suppose $r(s_t,a_t,\hat{g})=\mathbb{E}_{g\sim\pi_h,s,a\sim\mathbb{P}^{\pi_l,g}_t}r_l(s_t,a_t,g)/\mathcal{D}(g,\hat{g})$, thus we summate the $k$-step reward in the first subtask in part e as
\begin{align}
&\quad\quad\sum_{j=0}^{k-1}\mathbb{E}_{g\sim\pi_h,s,a\sim \mathbb{P}^{\pi_l,g}_{j}(\cdot,\cdot|s_0)}\Big[\gamma^jr(s_j,a_j,\hat{g})\Big]\nonumber \\
&=r(s_0,a_0,\hat{g})\sum_{j=0}^{k-1}\mathbb{E}_{g\sim\pi_h,s,a\sim \mathbb{P}^{\pi_l,g}_{j}(\cdot,\cdot|s_0)}\left[\gamma^j\frac{r(s_j,a_j,\hat{g})}{r(s_0,a_0,\hat{g})}\right]\nonumber\\
&=r(s_0,a_0,\hat{g})\sum_{j=0}^{k-1}\mathbb{E}_{g\sim\pi_h,s,a\sim \mathbb{P}^{\pi_l,g}_{j}(\cdot,\cdot|s_0)}\left[\gamma^j\frac{r_l(s_j,a_j,g)}{\mathcal{D}(g,\hat{g})}\frac{\mathcal{D}(g,\hat{g})}{r_l(s_0,a_0,g)}\right]\nonumber\\
&= r(s_0,a_0,\hat{g})\sum_{j=0}^{k-1}\mathbb{E}_{g\sim\pi_h,s,a\sim \mathbb{P}^{\pi_l,g}_{j}(\cdot,\cdot|s_0)}\left[\gamma^j\frac{r_l(s_j,a_j,g)}{r_l(s_0,a_0,g)}\right].
\end{align}

Since the low-level policy is trained as a goal-conditioned policy, we have $r_l(s_j,a_j,g)\leq r_l(s_k,a_k,g)$. And the summation in the first subtask can be
\begin{align}
&\quad\quad\sum_{j=0}^{k-1}\mathbb{E}_{g\sim\pi_h,s,a\sim \mathbb{P}^{\pi_l,g}_{j}(\cdot,\cdot|s_0)}\Big[\gamma^jr(s_j,a_j,\hat{g})\Big]\nonumber\\
&\leq r(s_0,a_0,\hat{g})\sum_{j=0}^{k-1}\mathbb{E}_{g\sim\pi_h,s,a\sim \mathbb{P}^{\pi_l,g}_{j}(\cdot,\cdot|s_0)}\left[\gamma^j\frac{r_l(s_k,a_k,g)}{r_l(s_0,a_0,g)}\right]\nonumber\\
&=r(s_0,a_0,\hat{g})\frac{1-\gamma^k}{1-\gamma}\frac{r_l(s_k,a_k,g)}{r_l(s_0,a_0,g)}.
\end{align}

Thus, we let the fraction $\mathcal{R}^{\pi_h,\pi_l}_i=r_l(s_k,a_k,g)/r_l(s_0,a_0,g)$ be the subgoal reachability definition, and the part e in $L_2$ can be
\begin{align}
\text{part e}&=\widetilde{V}^{\pi^*_h,\pi_l}_0(s_0)-\widetilde{V}^{\pi_h,\pi^*_l}_0(s_0)\nonumber\\
&=\sum_{j=0}^{k-1}\mathbb{E}_{g\sim\pi^*_h,s,a\sim \mathbb{P}^{\pi_l,g}_{j}(\cdot,\cdot|s_0)}\left[\gamma^jr(s_j,a_j,\hat{g})\right] - \sum_{j=0}^{k-1}\mathbb{E}_{g\sim\pi_h,s,a\sim \mathbb{P}^{\pi^*_l,g}_{j}(\cdot,\cdot|s_0)}\left[\gamma^jr(s_j,a_j,\hat{g})\right]\nonumber\\
&\leq r(s_0,a_0,\hat{g})\frac{1-\gamma^k}{1-\gamma}\mathcal{R}^{\pi^*_h,\pi_l}_0 - \sum_{j=0}^{k-1}\mathbb{E}_{g\sim\pi_h,s,a\sim \mathbb{P}^{\pi^*_l,g}_{j}(\cdot,\cdot|s_0)}\left[\gamma^jr(s_j,a_j,\hat{g})\right]\nonumber\\
&\leq r_{max}\frac{1-\gamma^k}{1-\gamma}\left(\mathcal{R}^{\pi_h,\pi_l}_0-\mathcal{R}^{\pi^*_h,\pi^*_l}_0\right) \quad \leftarrow \quad \text{$\Pi^*$ can achieve \textcolor{mydarkblue}{best subgoal reachability}}\nonumber \\
&\leq r_{max}\frac{1-\gamma^k}{1-\gamma}\mathcal{R}^{\pi_h,\pi_l}_0.
\end{align}

The penultimate inequality is based on the property of the induced optimal hierarchical policies. Compared with the learned $\pi_h$, Figure \textcolor{mydarkblue}{3} shows that $\pi^*_h$ can balance the subgoal reachability and the guidance, thus $\mathcal{R}^{\pi_h,\pi_l}_0\geq\mathcal{R}^{\pi^*_h,\pi_l}_0$ (note that the smaller $\mathcal{R}$ implies the better subgoal reachability). And, the optimal policies $\Pi^*$ can achieve the optimal subgoal reachability, i.e. $\mathcal{R}^{\pi^*_h,\pi^*_l}_0\leq\mathcal{R}^{\pi_h,\pi^*_l}_0$. Thus, we have $\left(\mathcal{R}^{\pi^*_h,\pi_l}_0-\mathcal{R}^{\pi_h,\pi^*_l}_0\right)\leq\left(\mathcal{R}^{\pi_h,\pi_l}_0-\mathcal{R}^{\pi^*_h,\pi^*_l}_0\right)$.

Then, we turn to part f in $L_2$. Considering the upper bound of the joint value function, we have
\begin{align}
\text{part f}&=\gamma^k\left[\mathbb{E}_{g\sim\pi^*_h,s\sim \mathbb{P}^{\pi_l,g}_{k}(\cdot|s_0)}V^{\pi_h,\pi^*_l}(s_k)-\mathbb{E}_{g\sim\pi_h,s\sim \mathbb{P}^{\pi^*_l,g}_{k}(\cdot|s_0)}V^{\pi_h,\pi^*_l}(s_k)\right]\nonumber\\
&\leq \gamma^k\int_{g\in\mathcal{G}}\int_{s\in\mathcal{S}}\left[\pi^*_h(g|s)-\pi_h(g|s)\right]\left[\mathbb{P}^{\pi_l,g}_{k}(s|s_0)-\mathbb{P}^{\pi^*_l,g}_{k}(s|s_0)\right]\frac{r_{max}}{1-\gamma}{\rm d}s{\rm d}g\nonumber\\
&\leq 2\gamma^k\int_{g\in\mathcal{G}}\int_{s\in\mathcal{S}}\frac{r_{max}}{1-\gamma}{\rm d}s{\rm d}g\nonumber\\
&=\frac{2\gamma^kr_{max}}{1-\gamma}.
\end{align}

With the derivation of part e and part f, we deal with part d by the recursion each $k$-steps as
\begin{align}
\text{part d}&=\gamma^k\mathbb{E}_{g\sim\pi^*_h,s\sim \mathbb{P}^{\pi_l,g}_{k}(\cdot|s_0)}\left[V^{\pi^*_h,\pi_l}(s_k)-V^{\pi_h,\pi^*_l}(s_k)\right]\nonumber\\
&\leq \gamma^{2k}\mathbb{E}_{g\sim\pi^*_h,s\sim \mathbb{P}^{\pi_l,g}_{2k}(\cdot|s_0)}\left[V^{\pi^*_h,\pi_l}(s_{2k})-V^{\pi_h,\pi^*_l}(s_{2k})\right]\nonumber\\
&\quad\quad + r_{max}\frac{\gamma^k-\gamma^{2k}}{1-\gamma}\mathcal{R}^{\pi_h,\pi_l}_1 + \frac{2\gamma^{2k}r_{max}}{1-\gamma}.
\end{align}

Thus, we combine the result of part d, part e and part f to obtain the results of $L_2$ as
\begin{align}
L_2 &= \text{part d} + \text{part e} + \text{part f}\nonumber\\
&\leq r_{max}\frac{1-\gamma^k}{1-\gamma}\mathcal{R}^{\pi_h,\pi_l}_0+r_{max}\frac{\gamma^k-\gamma^{2k}}{1-\gamma}\mathcal{R}^{\pi_h,\pi_l}_1+\frac{2\gamma^kr_{max}}{1-\gamma}+\frac{2\gamma^{2k}r_{max}}{1-\gamma}\nonumber\\
&\quad\quad +\gamma^{2k}\mathbb{E}_{g\sim\pi^*_h,s\sim \mathbb{P}^{\pi_l,g}_{2k}(\cdot|s_0)}\left[V^{\pi^*_h,\pi_l}(s_{2k})-V^{\pi_h,\pi^*_l}(s_{2k})\right] \nonumber \\
&\quad \ \ \vdots \nonumber\\
&\leq r_{max}\sum_{i=0}^{\infty}\frac{(1-\gamma^k)\gamma^{ik}}{1-\gamma}\mathcal{R}^{\pi_h,\pi_l}_i+\frac{2\gamma^{(i+1)k}}{1-\gamma}\nonumber \\
&\leq \frac{r_{max}}{(1-\gamma)^2}(\mathcal{R}^{\pi_h,\pi_l}_{max}+2\gamma^k).
\end{align}
In the last inequality, we define
\begin{equation}
\mathcal{R}^{\pi_h,\pi_l}_{max}=\max_{i\in\mathbb{N}}\mathcal{R}^{\pi_h,\pi_l}_i.
\end{equation}

Now, we have the results of $L_1$, $L_2$ and $L_3$. The performance difference bound between $\Pi^*$ and $\Pi$ can be obtained as
\begin{align}
V^{\Pi^*}(s_0)-V^{\Pi}(s_0)&=L_1+L_2+L_3\nonumber\\
&\leq 2r_{max}\frac{1+\gamma}{(1-\gamma)^2}\mathbb{E}_{g\sim\pi^*_h}\epsilon^g_{\pi^*_l,\pi_l} + \frac{r_{max}}{(1-\gamma)^2}(\mathcal{R}^{\pi_h,\pi_l}_{max}+2\gamma^k) \nonumber \\
&\quad\quad+ 2r_{max}\frac{1+\gamma}{(1-\gamma)^2}\mathbb{E}_{g\sim\pi_h}\epsilon^g_{\pi^*_l,\pi_l}\nonumber\\
&=\frac{2r_{max}}{(1-\gamma)^2}\left[(1+\gamma)\mathbb{E}_{g\sim\pi_h}\left(1+\frac{\pi^*_h}{\pi_h}\right)\epsilon^g_{\pi^*_l,\pi_l}+2\left(\mathcal{R}^{\pi_h,\pi_l}_{max} + 2\gamma^k\right)\right].
\end{align}
And the proof is complete.
\end{proof}

\begin{proposition}[Equivalence between $\pi^*$ and $\Pi^*$] 
With the $k$-step trajectory slicing and the alignment method, the performance of $\Pi^*$ and $\pi^*$ is equivalent, i.e., $V^{\pi^*}(s)=V^{\Pi^*}(s)$.
\end{proposition}
\begin{proof}
According to the $k$-step trajectory slicing and the alignment method, the induced optimal hierarchical policies $\Pi^*$ can be generated by aligning with the $k$-step trajectory slice derived by $\pi^*$, thus we have
\begin{align}
g_{(i+1)k}\sim\pi^*_h(\cdot|s_{ik})&=\mathbb{P}^{\pi^*}_k(s_{(i+1)k}|s_{ik})\nonumber \\
&=p(s_{ik})\prod_{j=0}^{k-1}P(s_{ik+j+1}|s_{ik+j},a_{ik+j})\pi^*(a_{ik+j}|s_{ik+j}),
\end{align}
\begin{align}
a_{ik+j}\sim\pi^*_l(\cdot|s_{ik+j},g_{(i+1)k})=\pi^*(a_{ik+j}|s_{ik+j}).
\end{align}
Thus, the value function for $\pi^*$ and the joint value function for $\Pi^*$ can be
\begin{align}
V^{\pi^*}(s_0)&=\sum_t^\infty\gamma^t\mathbb{E}_{s\sim p(s'|s,a),a\sim\pi^*}\left[r(s_t,a_t,\hat{g})\right]\nonumber\\
&=\sum_{i=0}^{\infty}\sum_{j=0}^{k-1}\mathbb{E}_{s\sim p(s'|s,a),a\sim\pi^*}\gamma^{ik+j}\left[r(s_{ik+j},a_{ik+j},\hat{g})\right]\nonumber\\
&=\sum_{i=0}^{\infty}\mathbb{E}_{g\sim\pi_h^*}\left\{\gamma^{ik}\sum_{j=0}^{k-1}\mathbb{E}_{s\sim p(s'|s,a),a\sim\pi_l^*}\gamma^j\left[r(s_{ik+j},a_{ik+j},\hat{g})\right]\right\}\nonumber\\
&=\sum_{i=0}^{\infty}\mathbb{E}_{g\sim \pi^*_h(\cdot|s)}\left[\gamma^{ik}\left(\sum_{j=0}^{k-1}\gamma^j\mathbb{E}_{s,a\sim \mathbb{P}^{\pi^*_l,g}_{ik+j}(\cdot,\cdot|s_0)}r(s_{ik+j},a_{ik+j},\hat{g})\right)\right]\nonumber\\
&=V^{\Pi^*}(s_0)
\end{align}

Thus, through the $k$-step trajectory slicing and the alignment method, the performance of $\Pi^*$ and $\pi^*$ is equivalent. And the proof is complete.
\end{proof}

\subsection{Useful Lemma and Assumption}
\begin{lemma}[Bellman Backup in HRL]\label{Bellman_Backup_HRL}
Consider that the joint value function can be decomposed by the summation of subtasks. Given the initial state $s_{ik}$ at the $i$-th subtask, the Bellman Backup of HRL defined in each subtask can be
\begin{align}
V^{\pi_h,\pi_l}(s_{ik})=\widetilde{V}_i^{\pi_h,\pi_l}(s_{ik})+\gamma^k\mathbb{E}_{g\sim\pi_h,s\sim \mathbb{P}^{\pi_l,g}_{(i+1)k}(\cdot|s_{ik})}\left[V^{\pi_h,\pi_l}(s_{(i+1)k})\right],
\end{align}
where $\widetilde{V}_i^{\pi_h,\pi_l}(s_{ik})$ is the the environment return of $\Pi$ with the $i$-th subtask, formulated as
\begin{equation}
\widetilde{V}_i^{\pi_h,\pi_l}(s_{ik})=\sum_{j=0}^{k-1}\mathbb{E}_{g\sim\pi_h,s,a\sim \mathbb{P}^{\pi_l,g}_{ik+j}(\cdot,\cdot|s_{ik})}\left[\gamma^jr(s_{ik+j},a_{ik+j},\hat{g})\right].
\end{equation}
\end{lemma}
\begin{proof}
According to the decomposition of the joint value function $V^{\pi_h,\pi_l}(s)$, we have
\begin{align}
V^{\pi_h,\pi_l}(s_0)&=\sum_{i=0}^{\infty}\mathbb{E}_{g\sim \pi_h}\left[\gamma^{ik}\left(\sum_{j=0}^{k-1}\gamma^j\mathbb{E}_{s,a\sim \mathbb{P}^{\pi_l,g}_{ik+j}(\cdot,\cdot|s_0)}r(s_{ik+j},a_{ik+j},\hat{g})\right)\right]\nonumber\\
&=\sum_{j=0}^{k-1}\mathbb{E}_{g\sim\pi_h,s,a\sim \mathbb{P}^{\pi_l,g}_{j}(\cdot,\cdot|s_0)}\left[\gamma^jr(s_{j},a_{j},\hat{g})\right]\nonumber\\
&\quad+\sum_{i=1}^{\infty}\mathbb{E}_{g\sim \pi_h}\left[\gamma^{ik}\left(\sum_{j=0}^{k-1}\gamma^j\mathbb{E}_{s,a\sim \mathbb{P}^{\pi_l,g}_{ik+j}(\cdot,\cdot|s_0)}r(s_{ik+j},a_{ik+j},\hat{g})\right)\right]\nonumber\\
&=\widetilde{V}_0^{\pi_h,\pi_l}(s_0)+\gamma^k\mathbb{E}_{g\sim\pi_h,s\sim \mathbb{P}^{\pi_l,g}_k(\cdot|s_k)}\left[V^{\pi_h,\pi_l}(s_k)\right].
\end{align}
Thus, we can conclude that
\begin{align}
V^{\pi_h,\pi_l}(s_{ik})=\widetilde{V}_i^{\pi_h,\pi_l}(s_{ik})+\gamma^k\mathbb{E}_{g\sim\pi_h,s\sim \mathbb{P}^{\pi_l,g}_{(i+1)k}(\cdot|s_{ik})}\left[V^{\pi_h,\pi_l}(s_{(i+1)})\right].
\end{align}
And the proof is complete.  
\end{proof}

\begin{lemma}[Markov chain TVD bound, time-varying]\label{TVD_distance}
Suppose the expected KL-divergence between two policy distributions is bounded as $\epsilon^g_{\pi^*_l,\pi_l}=\max_{s\in\mathcal{S},g\sim\pi_h}{\rm D}_{TV}\left(\pi^*_l(\cdot|s,g)\Vert\pi_l(\cdot|s,g)\right)$, and the initial state distributions are the same. Then, the distance in the state-action marginal is bounded as,
\begin{equation}
{\rm D}_{TV}\left(\mathbb{P}^{\pi^*_l,g}_t(\cdot,\cdot|s_0)\Big\Vert\mathbb{P}^{\pi_l,g}_t(\cdot,\cdot|s_0)\right)\leq t\epsilon^g_{\pi^*_l,\pi_l}
\end{equation}
\end{lemma}
\begin{proof}
Let $p(s'|s)$ as the adjacent state transition probability, which can be defined as
\begin{equation}
p(s'|s)=p(s)P(s'|s,a)\pi(a|s).
\end{equation}

Replacing the policy as the low-level policy $\pi_l$, we can derive the Markov chain TVD bound caused by the different low-level policy,
\begin{align}
&\quad\max_t\mathbb{E}_{s\sim p^t_1(s)}{\rm D}_{KL}(p_1(s'|s)\Vert p_2(s'|s))\nonumber\\
&=\max_t\mathbb{E}_{s\sim p^t_1(s)}p(s)P^a_{s,s'}(s'|s,a){\rm D}_{KL}(\pi^*_l(a|s,g)||\pi_l(a|s,g))\nonumber\\
&\leq\max_t\mathbb{E}_{s\sim p^t_1(s)}{\rm D}_{KL}(\pi^*_l(a|s,g)||\pi_l(a|s,g))\nonumber\\
&\leq\max_{s\in\mathcal{S},g\sim\pi_h}{\rm D}_{TV}\left(\pi^*_l(\cdot|s,g)\Vert\pi_l(\cdot|s,g)\right)\nonumber\\
&=\epsilon^g_{\pi^*_l,\pi_l}
\end{align}

Thus, follow the Lemma \textcolor{mydarkblue}{B.2} in~\cite{janner2019trust}, the distance in the state-action marginal is bounded as,
\begin{equation}
{\rm D}_{TV}\left(\mathbb{P}^{\pi^*_l,g}_t(\cdot,\cdot|s_0)\Big\Vert\mathbb{P}^{\pi_l,g}_t(\cdot,\cdot|s_0)\right)\leq t\epsilon^g_{\pi^*_l,\pi_l}.
\end{equation}
And the proof is complete.
\end{proof}

\begin{assumption}[Refer to Assumption~\textcolor{mydarkblue}{1} in \cite{zhang2022adjacency}]
For all $s\in\mathcal{S}$ and $g\in\mathcal{G}$, the environmental reward can be written as
\begin{equation}
r(s,a,\hat{g})=\sum_{s'}P^a_{s,s'}(s'|s,a)\pi_l(a|s,g)\widetilde{r}(s,s')=\mathbb{E}_{g\sim\pi_h,s,a\sim\mathbb{P}^{\pi_l,g}_t}r_l(s_t,a_t,g)/\mathcal{D}(g,\hat{g}).
\end{equation}
where $\widetilde{r}: \mathcal{S}\times\mathcal{G}\rightarrow[0,r_{max}]$ is a state-reachability reward function. 
\label{reward_assumption}
\end{assumption}
In this assumption, the subgoal $g$ generated by the high-level policy represents the desired state to be reached, while the intermediate low-level state and action details are controlled by the low-level policy. Therefore, considering that the subgoals are generated towards the environmental goal $\hat{g}$, when given a low-level optimal/learned policy, it is natural to assume that the $k$-step stage reward only depends on the state where the agent starts and the state where the agent arrives.

\clearpage
\section{Experimental Details}
\subsection{Implementation Details}
Our method BrHPO and all baselines are implemented based on PyTorch.

\paragraph{BrHPO.} We employ the soft actor-critic (SAC) algorithm~\cite{haarnoja2018soft} as the backbone framework for both high- and low-level policies. For the high-level policy, considering that the subtask trajectory $\tau^{\pi_h,\pi_l}_i$ in each subtask would be abstracted as one transition in high level, we convert the trajectory $(s_{ik:(i+1)k-1}, a_{ik:(i+1)k-1}, g_{(i+1)k}, r_{h, ik}, s_{(i+1)k})$ into a high-level transition tuple $(s_{ik}, g_{(i+1)k}, r_{h, ik}, s_{(i+1)k})$. Then, when a subtask ends, we compute the subgoal reachability by 
\begin{equation*}
\mathcal{R}^{\pi_h,\pi_l}_i=\mathbb{E}_{r_l\sim\tau^{\pi_h,\pi_l}_i}\frac{r_{l,(i+1)k}}{r_{l,ik}}.
\end{equation*}
Then, we can optimize the high-level policy by
\begin{equation*}
Q^{\pi_h}(s,g)=\arg\min_Q\frac{1}{2}\mathbb{E}_{s,g\sim D_h}\left[r_h(s,g)+\gamma\mathbb{E}_{s'\sim D_h,g'\sim\pi_h}Q^{\pi_h}(s',g')-Q^{\pi_h}(s,g)\right]^2,
\end{equation*}
\begin{equation*}
\pi_h=\arg\min_{\pi_h}\mathbb{E}_{s\sim D_h}\left[{\rm D}_{KL}(\pi_h(\cdot|s)\Vert \exp(Q^{\pi_h}(s,g)-V^{\pi_h}(s)))+\lambda_1\mathcal{R}^{\pi_h,\pi_l}_i\right].
\end{equation*}

For the low-level policy which can be trained as a goal-conditioned one, we design the reachability-aware low-level policy as 
\begin{equation*}
\hat{r}_l(s_{ik+j},a_{ik+j},g_{(i+1)k})=r_l(s_{ik+j},a_{ik+j},g_{(i+1)k})-\lambda_2\mathcal{R}^{\pi_h,\pi_l}_i.
\end{equation*}
The training tuples for the low-level policy are formed as the per-step state-action transitions $(s_{ik+j},g_{(i+1)k},a_{ik+j},r_{l,ik+j},s_{ik+j+1},g_{(i+1)k})$\footnote[1]{We use the absolute subgoal in this paper, that is, $g_{(i+1)k}=s_{ik}+\pi_h(\cdot|s_{ik})$.}, which then are stored in the low-level buffer $D_l$. Thus, with the training tuples, we can optimize the low-level policy as,
\begin{equation*}
Q^{\pi_l}(s,a)=\arg\min_Q\frac{1}{2}\mathbb{E}_{s,g,a\sim D_l}\left[\hat{r}_l(s,a,g)+\gamma\mathbb{E}_{s',g\sim D_l,a'\sim\pi_l}Q^{\pi_l}(s',a')-Q^{\pi_l}(s,a)\right]^2,
\end{equation*}
\begin{equation*}
\pi_l=\arg\min_{\pi_l}\mathbb{E}_{s,g\sim D_l}\left[{\rm D}_{KL}(\pi_l(\cdot|s,g)\Vert \exp(Q^{\pi_l}(s,a)-V^{\pi_l}(s)))\right].
\end{equation*}

\paragraph{Algorithm framework.}
We briefly give an overview of our proposed BrHPO in algorithm~\ref{algo_BrHPO}. Notably, the mutual response mechanism effectively calculates the subgoal reachability for bilateral information and then incorporates it into hierarchical policy optimization for mutual error correction, promoting performance and reducing computation load. 
\begin{algorithm}[ht]
      \caption{Bidirectional-reachable Hierarchical Policy Optimization (BrHPO)}
      \label{algo_BrHPO}
      \begin{algorithmic}
          \State \textbf{initialize:} policy networks $\pi_h$, $\pi_l$, $Q$-networks $Q^{\pi_h}$, $Q^{\pi_l}$, replay buffers for high-level $D_h$ and low-level $D_l$
          \For{each training episode}
          \While{not \textit{done}}
          \State sample subgoals $g\sim\pi_h(\cdot|s)$
            \For{each step in a subtask}
              \State Sample actions $a\sim\pi_l(\cdot|s,g)$
              \State Store $(s,g,a,r_l,s',g)$ into a temp buffer
              \State Update $\pi_l$ by (\ref{critic_low}) and (\ref{actor_low}) from $D_l$ \Comment{\textcolor{myblue}{low-level policy optimization}}
            \EndFor
            \State Calculate $\mathcal{R}^{\pi_h,\pi_l}_i$ by (\ref{appro_subgoal_reachability}) \Comment{\textcolor{myred}{subgoal reachability computation}}
            \State Compute $\hat{r}_l$ by (\ref{modified_low_reward}) and push the tuples in $D_l$ \Comment{\textcolor{myblue}{reachability-aware low-level reward}}
            \State Store $(s,g,r_h,s',\hat{\mathcal{R}}^{\pi_h,\pi_l}_i)$ into $D_h$
            \State Update $\pi_h$ by (\ref{critic_high}) and (\ref{actor_high}) from $D_l$ \Comment{\textcolor{mygreen}{high-level policy optimization}}
          \EndWhile
        \EndFor
      \end{algorithmic}
\end{algorithm}

\paragraph{HIRO.}  
In this work~\cite{nachum2018data}, to deal with the non-stationarity, where old off-policy experience may exhibit different transitions conditioned on the same goals, they heuristically relabel the subgoal $\tilde{g}$ as
\begin{equation*}
    \log \mu^{lo}(a_{t:t+c+1}|s_{t:t+c+1},\tilde{g}_{t:t+c+1})\propto-\frac{1}{2}\sum_{i=t}^{t+c-1}\|a_i-\mu^{lo}(s_i,\tilde{g}_i)\|^2_2+\text{const}.
\end{equation*}
To solve this problem efficiently, they calculated the quantity on eight candidate goals sampled randomly from a Gaussian centred at $s_{t+c}-s_t$. Then, with the correcting high-level experience, the high-level policy can be optimized by off-policy methods. Compared with our methods, the off-correction can be regarded as a low-level domination method, which requires the high-level experience to be modified by the subgoal reachability demonstrated at a low level.

\paragraph{HIGL.}
In this work~\cite{kim2021landmark}, to restrict the high-level action space from the whole goal space to a $k$-step adjacent region, they introduced the shortest transition distance as a constraint in high-level policy optimization. Besides, they utilized farthest point sampling and priority queue $\mathcal{Q}$ to improve the subgoal coverage and novelty. To enhance the subgoal reachability, they made pseudo landmark be placed between the selected landmark and the current state in the goal space as follows:
\begin{equation*}
g_t^{\text{pseudo}}:=g_t^{\text{cur}}+\delta_{\text{pseudo}}\cdot\frac{g_t^{\text{sel}}-g_t^{\text{cur}}}{\|g_t^{\text{sel}}-g_t^{\text{cur}}\|^2}.
\end{equation*}
To establish the adjacency constraint by the shortest transition distance, they refer to HRAC~\cite{zhang2020generating} and adopt an adjacent matrix to model it. Specifically, we note that the performance of HIGL in the AntMaze task is different from the original report in their paper, mainly due to the different scales. Thus, we set the same scale for all tasks for fairness. To ensure that HIGL performs well in these tasks, we adjusted hyper-parameters such as "landmark coverage" and "n landmark novelty".

\begin{figure}[!h]
\centering
\includegraphics[width=0.6\textwidth]{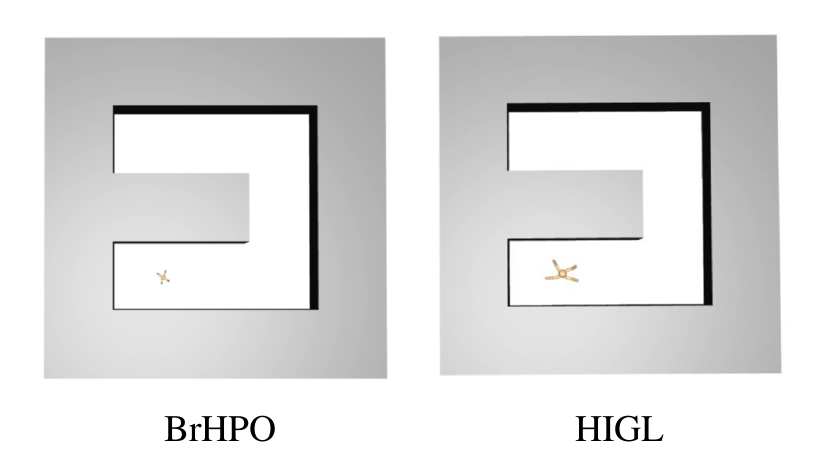} 
\caption{Comparison of the scales in the maze tasks between BrHPO and HIGL.}
\label{HIGH_antmaze}
\end{figure}

\paragraph{CHER.} This work~\cite{kreidieh2019inter} proposed a cooperation framework for HRL. In this work,  the HRL problem can similarly be framed as a constrained optimization problem, 
\begin{equation*}
\max_{\pi_m}\left[J_m+\min_{\lambda\leq0}\left(\lambda\delta-\lambda\min_{\pi_w}J_w\right)\right].
\end{equation*}
To deal with this problem, they update the high- ($\pi_m$) and low-level ($\pi_w$) policies by
\begin{equation*}
\theta_w\leftarrow\theta_w+\alpha\nabla_{\theta_w}J_w, \quad\text{and}, \quad \theta_m\leftarrow\theta_m+\alpha\nabla_{\theta_w}(J_w+\lambda J_w).
\end{equation*}
Compared to CHER, our BrHPO method distinguishes itself in several key aspects. In CHER, hierarchical cooperation is achieved solely through high-level policy optimization, while the low-level policy is trained as a generally goal-conditioned policy without further improvement. Moreover, the high-level optimization in CHER introduces $J_w$ as $(J_w+\lambda J_w)$, necessitating a focus on the step-by-step behaviour of the low-level policy.

In contrast, our BrHPO method incorporates the concept of subgoal reachability, which considers the initial and final states of the subtasks. This design choice empowers the high-level policy to relax the exploration burden on the low-level policy. By leveraging subgoal reachability, our approach enables more efficient exploration of the low-level policy and facilitates effective hierarchical cooperation between the high-level and low-level policies.
 
\paragraph{RIS.}
In this work~\cite{chane2021goal}, based on the hindsight method, they collected feasible state trajectories and predicted an appropriate distribution of imagined subgoals. They first defined subgoals $s_g$ as midpoints on the path from the current state $s$ to the goal $g$, and further minimized the length of the paths from $s$ to $s_g$ and $s_g$ to $g$. Thus, the high-level policy can be updated as
\begin{equation*}
\pi_{k+1}^H=\arg\min_{\pi^H}\mathbb{E}_{(s,g)\sim D, s_g\sim\pi^H(\cdot|s,g)}[C_\pi(s_g|s,g)].
\end{equation*}
Then, with the imagined subgoals, the low-level policy can be trained by
\begin{equation*}
\pi_{\theta_{k+1}}=\arg\max_{\theta}\mathbb{E}_{(s,g)\sim D}\mathbb{E}_{a\sim\pi_\theta(\cdot|s,g)}\left[Q^{\pi}(s,a,g)-\alpha D_{KL}\left(\pi_\theta\Vert\pi^{prior}_k\right)\right].
\end{equation*}

\subsection{Network Architecture}
For the hierarchical policy network, we employ SAC~\cite{haarnoja2018soft} as both the high-level and the low-level policies. Each actor and critic network for
both high level and low level consists of 3 fully connected layers with ReLU nonlinearities. The size of each hidden layer is $(256, 256)$. The output of the high- and low-level actors is activated using the linear function and is scaled to the range of corresponding action space. 

We use Adam optimizer~\cite{kingma2014adam} for all networks in BrHPO.

\subsection{Environmental Setup}
We adopt six challenging long-term tasks to evaluate BrHPO, which can be categorized into the \textit{dense} case and the \textit{sparse} case. For the maze navigation tasks, a simulated ant starts at $(0,0)$ and the environment reward is defined as $r=-\sqrt{(x-g_x)^2+(y-g_y)^2}$ (except for AntFall, $r=-\sqrt{(x-g_x)^2+(y-g_y)^2+(z-g_z)^2}$). While in the robotics manipulation tasks, a manipulator is initialized with a horizontal stretch posture. The environmental reward is defined as a binary one, determined by the distance between the  end-effector (or the object in Pusher) and the target point
\begin{equation}
  r=
  \begin{cases}
    -1, & d > 0.25, \\
    0, & d \leq 0.25.
  \end{cases}
\end{equation}
And, the success indicator is defined as whether the final distance is less than a pre-defined threshold, where the maze navigation tasks require $d<5$ and the robotics manipulation tasks require $d<0.25$.

\paragraph{AntMaze.} 
A simulated eight-DOF ant starts from the left bottom $(0,0)$ and needs to approach the left top corner $(0,16)$. At each training episode, a target position is sampled uniformly at random from $g_x\sim[-4,20],g_y\sim[-4,20]$. At the test episode, the target points are fixed at $(g_x,g_y)=(0,16)$.

\paragraph{AntBigMaze.}
Similar to the AntMaze task, we design a big maze to evaluate the exploration capability of BrHPO. In particular, the target position is chosen randomly from one of $(g_x,g_y)=(32,8)$ and $(g_x,g_y)=(66,0)$, which makes it harder to find a feasible path.

\paragraph{AntPush.}
A movable block at $(0,8)$ is added to this task. The ant needs to move to the left side of the block and push it into the right side of the room, for a chance to reach the target point above, which requires the agent to avoid training a greedy algorithm. At each episode, the target position is fixed to $(g_x,g_y)=(0,19)$.

\paragraph{AntFall.}
In this task, the agent is initialized on a platform of height 4. Like the AntPush environment, the ant has to push a movable block at $(8,8)$ into a chasm to create a feasible road to the target, which is on the opposite side of the chasm, while a greedy policy would cause the ant to walk towards the target and fall into the chasm. At each episode, the target position is fixed to $(g_x,g_y,g_z)=(0,27,4.5)$.

\paragraph{Reacher3D.}
A simulated 7-DOF robot manipulator needs to move its end-effector to a desired position. The initial position of the end-effector is at $(0,0,0)$ while the target is sampled from a Normal distribution with zero mean and 0.1 standard deviation.

\paragraph{Pusher.}
Pusher additionally includes a puck-shaped object based on the Reacher3D task, and the end-effector needs to find the object and push it to a desired position. At the initialization, the object is placed randomly and the target is fixed at $(g_x,g_y,g_z)=(0.45 , -0.05 , -0.323)$.

We summarise these six tasks in Table~\ref{env_para}.
\begin{table}[!h]
  \caption{Overview on Environment settings.}
  \label{environment_setting}
  \begin{center}
    \resizebox{1.0\textwidth}{!}{
      \begin{tabular}{
          >{\centering}m{0.15\textwidth}
          | c
          | c
          | c
          | c
          | c
      }
      \toprule
      Environment & state & action & environment reward & episode step & success indicator \\
      \midrule
      AntMaze & 32 & 8 & negative x-y distance & 500 & $r_{\text{final}}\geq -5$\\
      \midrule
      AntBigMaze & 32 & 8 & negative x-y distance & 1000 & $r_{\text{final}}\geq -5$\\
      \midrule
      AntPush & 32 & 8 & negative x-y distance & 500 & $r_{\text{final}}\geq -5$\\
      \midrule
      AntFall & 33 & 8 & negative x-y-z distance & 500 & $r_{\text{final}}\geq -5$\\
      \midrule
      Reacher3D & 20 & 7 & negative x-y-z distance & 100 & $d_{\text{final}}\leq 0.25$\\
      \midrule
      Pusher & 23 & 7 & negative x-y-z distance & 100 & $d_{\text{final}}\leq 0.25$\\
      \bottomrule
      \end{tabular}
    }
  \end{center}
  \label{env_para}
\end{table}

\clearpage
\subsection{Hyper-parameters}
Table~\ref{hyperparameters} lists the hyper-parameters used in training BrHPO over all tasks.
\begin{table}[!h]
    \caption{The hyper-parameters settings for BrHPO.}
    \label{hyperparameters}
    \begin{center}
      \resizebox{1.0\textwidth}{!}{
        \begin{tabular}{
            >{\centering}m{0.15\textwidth}
            | c
            | c
            | c
            | c
            | c
            | c
        }
            \toprule
            & AntMaze & AntBigMaze & AntPush & AntFall & Reacher3D & Pusher\\
            \midrule
            $Q$-value network (both high and low) & \multicolumn{6}{c}{
              MLP with hidden size 256
            } \\
            \midrule
            policy network (both high and low) & \multicolumn{6}{c}{
              Gaussian MLP with hidden size 256
            } \\
            \midrule
            discounted factor $\gamma$ & \multicolumn{6}{c}{0.99} \\
            \midrule
            soft update factor $\tau$ & \multicolumn{6}{c}{0.005} \\
            \midrule
            $Q$-network learning rate & \multicolumn{6}{c}{0.001} \\
            \midrule
            policy network learning rate & \multicolumn{6}{c}{0.0001} \\
            \midrule
            automatic entropy tuning (high-level) & \multicolumn{2}{c|}{False} & \multicolumn{2}{c|}{True} & \multicolumn{2}{c}{False} \\
            \midrule
            automatic entropy tuning (low-level) & \multicolumn{6}{c}{False} \\
            \midrule
            batch size & \multicolumn{6}{c}{128} \\
            \midrule
            update per step & \multicolumn{6}{c}{1} \\
            \midrule
            target update interval & \multicolumn{6}{c}{2} \\
            \midrule
            high-level replay buffer & \multicolumn{6}{c}{1e5} \\
            \midrule
            low-level replay buffer & \multicolumn{6}{c}{1e6} \\
            \midrule
            start steps & \multicolumn{6}{c}{5e3} \\
            \midrule
            subtask horizon & \multicolumn{4}{c|}{20} & \multicolumn{2}{c}{10} \\
            \midrule
            reward scale & \multicolumn{6}{c}{1} \\
            \midrule
            high-level responsive factor $\lambda_1$ & \multicolumn{2}{c|}{2} & \multicolumn{2}{c|}{0.5} & \multicolumn{2}{c}{2}\\
            \midrule
            low-level responsive factor $\lambda_2$ & \multicolumn{4}{c|}{10} & \multicolumn{2}{c}{5} \\
            \bottomrule
        \end{tabular}
      }
    \end{center}
\end{table}

\section{Additional experiments}\label{sec:additional_results}
\paragraph{Additional Metrics.}
We report additional~(aggregate) performance metrics of BrHPO and other baselines on the six tasks using the \texttt{rliable} toolkit~\cite{agarwal2021deep}. As show in Figure~\ref{fig:rliable}, BrHPO outperforms other baselines in terms of Median, interquantile mean~(IQM), Mean and Optimality Gap results.
\begin{figure}[!h]
  \centering
  \includegraphics[width=\textwidth]{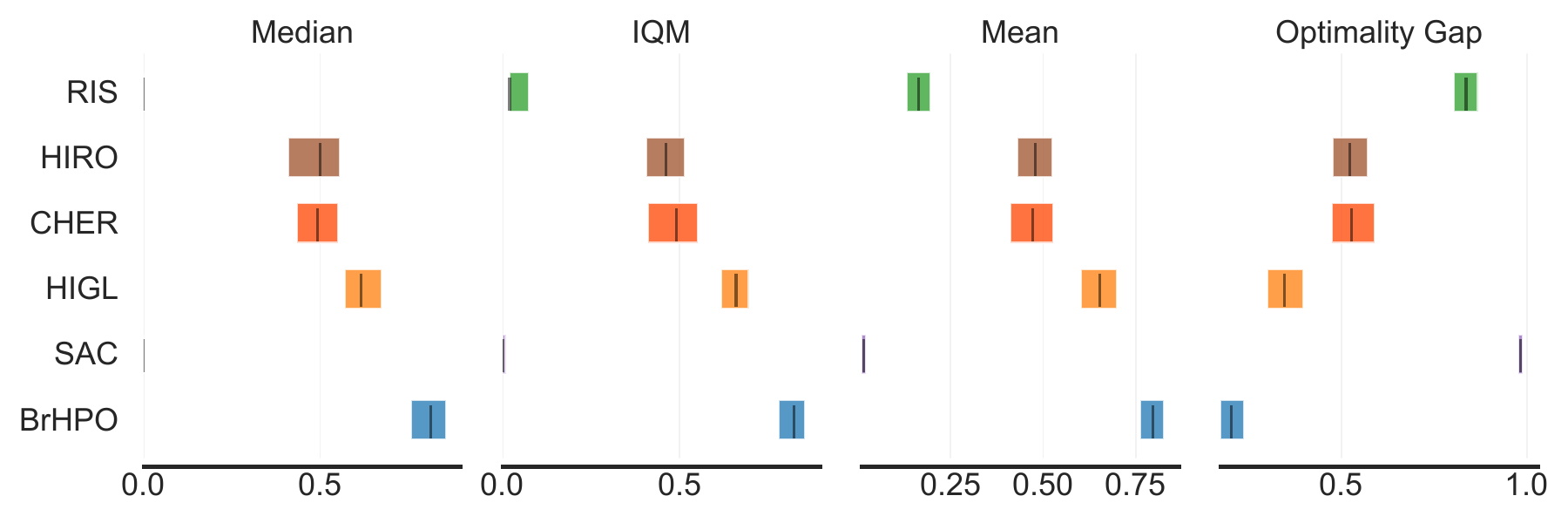}
  \caption{Median, IQM, Mean~(higher values are better) and Optimality Gap~(lower values are better) performance of BrHPO and all baselines on six tasks.}
  \label{fig:rliable}
\end{figure} 

\paragraph{Subgoal reachability report.} 
We report the average subgoal reachability $\mathcal{R}^{\pi_h,\pi_l}_i$ of each environment by Table~\ref{subgoal_reachability_table}. Note that, the value $\mathcal{R}^{\pi_h,\pi_l}_i\rightarrow 0$ means the final distance $\mathcal{D}(\psi(s_{(i+1)k}),g_{(i+1)k})\rightarrow 0$, thus implying the better subgoal reachability. From the results, our implementation is simple yet effective, which can improve subgoal reachability significantly. Besides, the results show that when there are contact dynamics in the environment, such as AntPush, AntFall and Pusher, the subgoal reachability may be decreased, which inspires us to further develop investigation in these cases. 
 
\begin{table}[!h]
  \caption{The average subgoal reachability of BrHPO.}
  \label{subgoal_reachability_table}
  \begin{center}
    \resizebox{1.0\textwidth}{!}{
      \begin{tabular}{
          >{\centering}m{0.15\textwidth}
          | c
          | c
          | c
          | c
          | c
          | c
      }
          \toprule
          Environment & AntMaze & AntBigMaze & AntPush & AntFall & Reacher3D & Pusher \\
          \midrule
          subgoal reachability
          & 0.22
          & 0.29
          & 0.33
          & 0.32
          & 0.13
          & 0.18
          \\
          \bottomrule
      \end{tabular}
    }
  \end{center}
\end{table}

\paragraph{Ablation by the sparse environment.}
Additionally, we provide ablation studies conducted on the Reacher3D task (\textbf{\textit{sparse}}) instead of the AntMaze task (\textbf{\textit{dense}}). We investigate the effectiveness of the mutual response mechanism by 1) the three variants of BrHPO, containing \textbf{\textit{Vanilla}}, \textbf{\textit{NoReg}} and \textbf{\textit{NoBonus}}, and 2) the weighted factors $\lambda_1$ and $\lambda_2$. We show the results in Figure~\ref{Reacher3D_ablation_variants}. Overall, we find that the tendency from the Reacher3D task is similar to the AntMaze task, which verifies the effectiveness of our BrHPO in the \textit{sparse} reward case.

\begin{figure}[!h]
  \centering
  \begin{subfigure}[t]{0.32\textwidth}
    \centering
    \includegraphics[width=\textwidth]{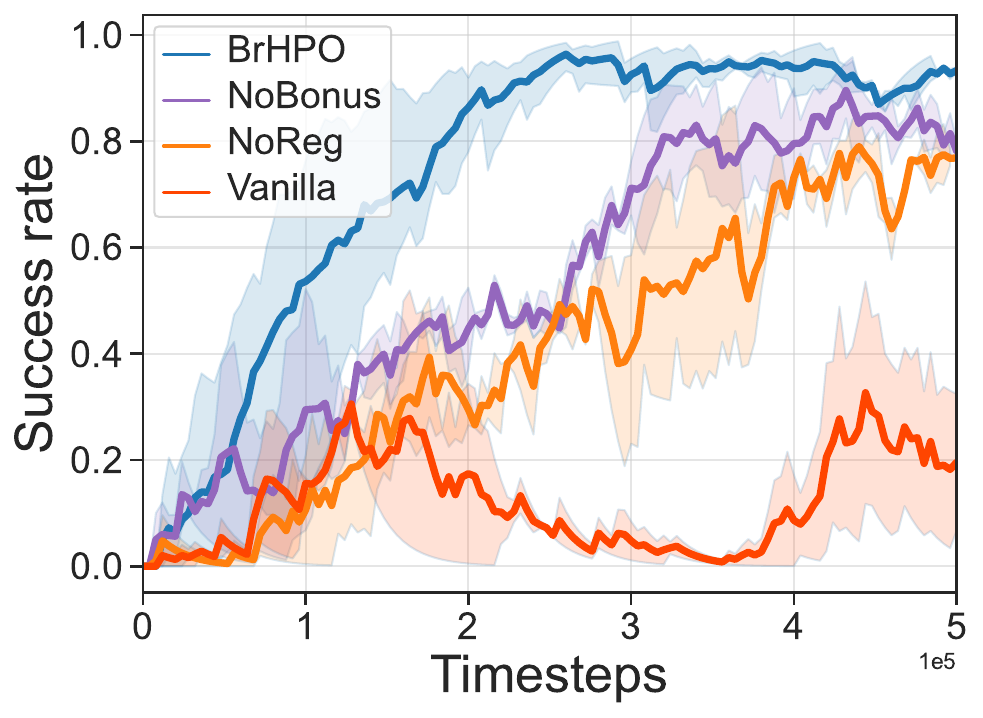}
    \caption{Diff. variants}
  \end{subfigure}
  \begin{subfigure}[t]{0.32\textwidth}
    \centering
    \includegraphics[width=\textwidth]{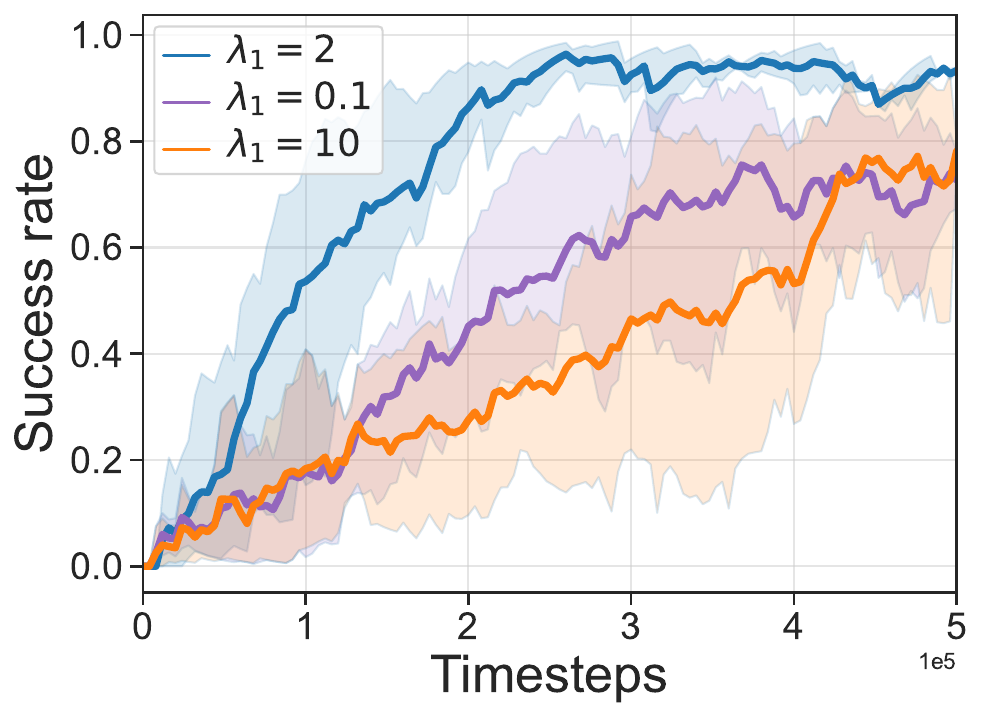}
    \caption{Diff. $\lambda_1$}
  \end{subfigure}
  \hfill
  \begin{subfigure}[t]{0.32\textwidth}
    \centering
    \includegraphics[width=\textwidth]{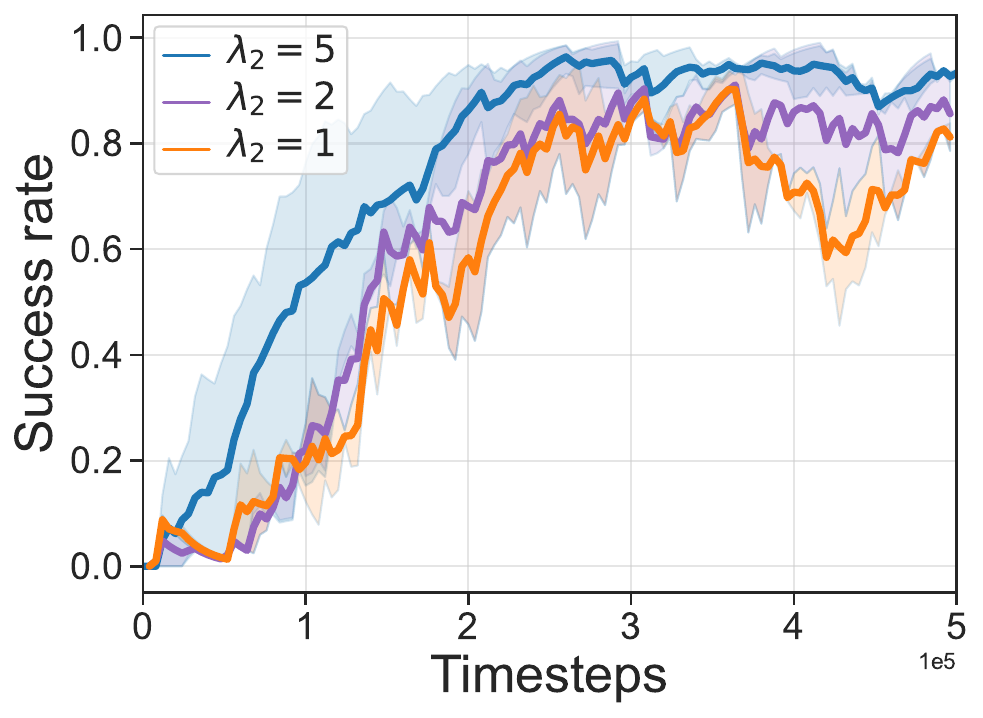}
    \caption{Diff. $\lambda_2$}
  \end{subfigure}
  \caption{The ablation of mutual response mechanism by Reacher3D task. Mean and std by 4 runs.}
  \label{Reacher3D_ablation_variants}
\end{figure}

\newpage
\paragraph{Empirical study in stochastic environments.}
To empirically verify the stochasticity robustness of BrHPO, we utilize it the a set of stochastic tasks, including stochastic AntMaze, AntPush and Reacher3D, which are modified from the original tasks. Referring to HRAC~\cite{zhang2020generating}, we interfere with the position of the ant (x,y) and the position of the end-effector (x,y,z) with Gaussian noise of different standard deviations, including $\sigma=0.01$, $\sigma=0.05$ and $\sigma=0.1$, to verify the robustness against the increasing environmental stochasticity. As shown in Figure~\ref{stochastic_ablation_variants}, BrHPO can achieve similar asymptotic performance with different noise magnitudes in stochastic AntMaze, AntPush and Reacher3D, which shows the robustness to stochastic environments.

\begin{figure}[!h]
  \centering
  \begin{subfigure}[t]{0.32\textwidth}
    \centering
    \includegraphics[width=\textwidth]{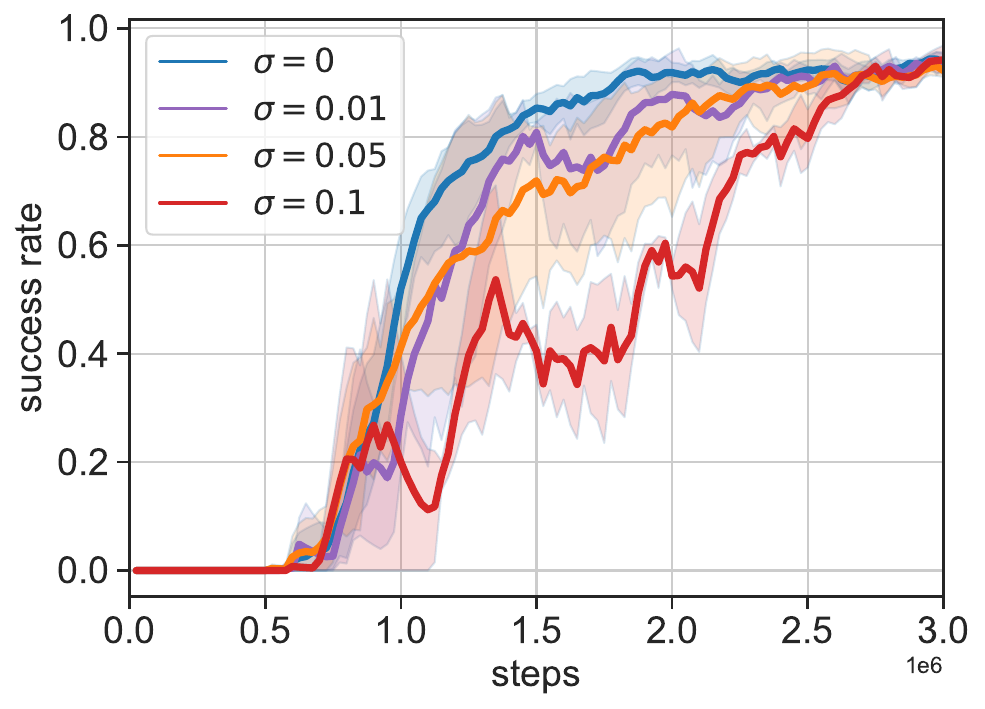}
    \caption{Stochastic AntMaze}
  \end{subfigure}
  \begin{subfigure}[t]{0.32\textwidth}
    \centering
    \includegraphics[width=\textwidth]{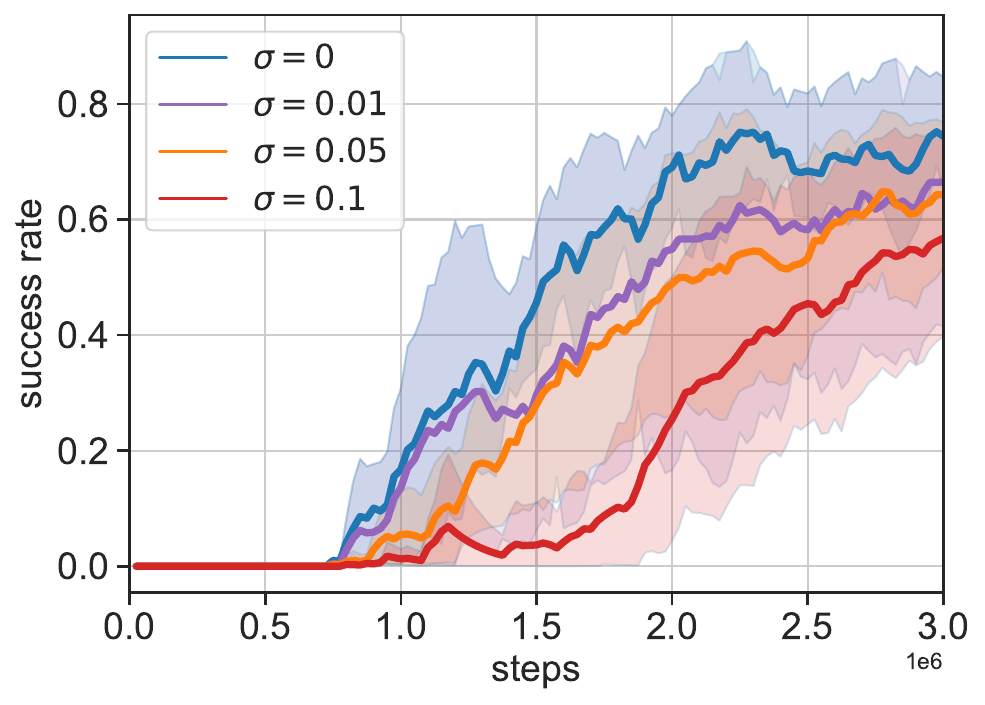}
    \caption{Stochastic AntPush}
  \end{subfigure}
  \hfill
  \begin{subfigure}[t]{0.32\textwidth}
    \centering
    \includegraphics[width=\textwidth]{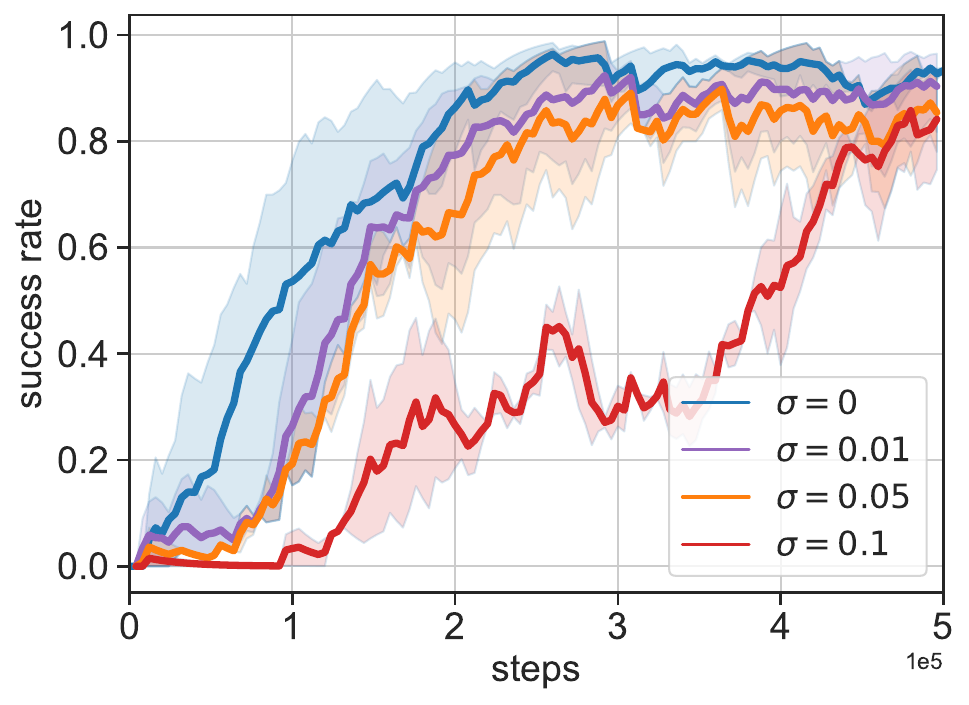}
    \caption{Stochastic Reacher3D}
  \end{subfigure}
  \caption{The empirical evaluation of BrHPO by stochastic environments. Mean and std by 4 runs.}
  \label{stochastic_ablation_variants}
\end{figure}

\subsection{Computing Infrastructure and Training Time}\label{appendix:training_time}
For completeness, we list the computing infrastructure and benchmark training times for BrHPO and all baselines by Table~\ref{Computation_time}. As discussed in section~\ref{Ablation}, the training complexity of BrHPO is much less than other HRL methods, which can be comparable to the flat policy.
\begin{table}[!h]
  \caption{Computing infrastructure and training time on each task (in hours).}
  \label{Computation_time}
  \begin{center}
      \begin{tabular}{
          >{\centering}m{0.15\textwidth}
          | c
          | c
          | c
          | c
          | c
          | c
      }
          \toprule
          & AntMaze & AntBigMaze & AntPush & AntFall & Reacher3D & Pusher \\
          \midrule
          CPU & \multicolumn{6}{c}{
              AMD EPYC™ 7763
          } \\
          \midrule
          GPU & \multicolumn{6}{c}{
              NVIDIA GeForce RTX 3090
          } \\
          
          \midrule
          HIRO
          & 16.66
          & 23.14
          & 18.29
          & 25.43
          & 3.42
          & 4.25
          \\
          \midrule
          HIGL
          & 31.59
          & 48.45
          & 30.95
          & 49.60
          & 5.96
          & 7.05
          \\
          \midrule
          CHER
          & 15.38
          & 20.53
          & 16.71
          & 21.37
          & 2.96
          & 3.16
          \\
          \midrule
          RIS
          & 40.83
          & 53.49
          & 38.46
          & 57.05
          & 8.63
          & 9.88
          \\
          \midrule
          \textcolor{mygreen}{\textbf{SAC}}
          & \textcolor{mygreen}{\textbf{10.57}}
          & \textcolor{mygreen}{\textbf{11.36}}
          & \textcolor{mygreen}{\textbf{11.75}}
          & \textcolor{mygreen}{\textbf{15.64}}
          & \textcolor{mygreen}{\textbf{2.35}}
          & \textcolor{mygreen}{\textbf{2.68}}
          \\
          \midrule
          \textcolor{myblue}{\textbf{BrHPO}}
          & \textcolor{myblue}{\textbf{12.75}}
          & \textcolor{myblue}{\textbf{18.74}}
          & \textcolor{myblue}{\textbf{13.43}}
          & \textcolor{myblue}{\textbf{19.17}}
          & \textcolor{myblue}{\textbf{2.73}}
          & \textcolor{myblue}{\textbf{3.53}}
          \\
          \midrule
          \textcolor{myred}{\textbf{comparison
          (Ours - SAC)}}
          & \textcolor{myred}{\textbf{2.18}}
          & \textcolor{myred}{\textbf{7.38}}
          & \textcolor{myred}{\textbf{1.68}}
          & \textcolor{myred}{\textbf{3.53}}
          & \textcolor{myred}{\textbf{0.38}}
          & \textcolor{myred}{\textbf{0.85}}
          \\
          \bottomrule
      \end{tabular}
  \end{center}
\end{table}

\end{document}